\documentclass{article}

\usepackage{microtype}
\usepackage{graphicx}
\usepackage{subfigure}
\usepackage{booktabs} % for professional tables

\usepackage{hyperref}

\usepackage[accepted]{icml2024}

\usepackage{amsmath}
\usepackage{amssymb}
\usepackage{mathtools}
\usepackage{amsthm}
\usepackage[capitalize,noabbrev]{cleveref}
\usepackage{custom_commands}

\usepackage{thmtools, thm-restate}

\theoremstyle{plain}
\newtheorem{theorem}{Theorem}[section]

\newtheorem{lemma}[theorem]{Lemma}

\theoremstyle{definition}
\newtheorem{definition}[theorem]{Definition}
\newtheorem{assumption}[theorem]{Assumption}
\theoremstyle{remark}

\usepackage[textsize=tiny]{todonotes}

\icmltitlerunning{Robust Multi-Task Learning with Excess Risks}

\begin{document}

\twocolumn[
\icmltitle{Robust Multi-Task Learning with Excess Risks}



\icmlsetsymbol{equal}{*}

\begin{icmlauthorlist}
\icmlauthor{Yifei He}{uiuc}
\icmlauthor{Shiji Zhou}{thu}
\icmlauthor{Guojun Zhang}{waterloo}
\icmlauthor{Hyokun Yun}{amazon}
\icmlauthor{Yi Xu}{amazon}
\icmlauthor{Belinda Zeng}{amazon}
\icmlauthor{Trishul Chilimbi}{amazon}
\icmlauthor{Han Zhao}{uiuc,amazon}
\end{icmlauthorlist}

\icmlaffiliation{uiuc}{Department of Computer Science, University of Illinois Urbana-Chamapign, Urbana, IL, USA}
\icmlaffiliation{amazon}{Amazon, Seattle, WA, USA}
\icmlaffiliation{waterloo}{School of Computer Science, University of Waterloo, Waterloo, ON, Canada}
\icmlaffiliation{thu}{Department of Automation, Tsinghua University, Beijing, China}

\icmlcorrespondingauthor{Yifei He}{yifeihe3@illinois.edu}
\icmlcorrespondingauthor{Han Zhao}{hanzhao@illinois.edu}

\icmlkeywords{Machine Learning, ICML}

\vskip 0.3in
]



\printAffiliationsAndNotice{}  % leave blank if no need to mention equal contribution

\begin{abstract}
 Multi-task learning (MTL) considers learning a joint model for multiple tasks by optimizing a convex combination of all task losses. To solve the optimization problem, existing methods use an adaptive weight updating scheme, where task weights are dynamically adjusted based on their respective losses to prioritize difficult tasks. However, these algorithms face a great challenge whenever \textit{label noise} is present, in which case excessive weights tend to be assigned to noisy tasks that have relatively large Bayes optimal errors, thereby overshadowing other tasks and causing performance to drop across the board. To overcome this limitation, we propose \textbf{M}ulti-\textbf{T}ask \textbf{L}earning with \textbf{Excess} Risks (ExcessMTL), an excess risk-based task balancing method that updates the task weights by their distances to convergence instead. Intuitively, ExcessMTL assigns higher weights to worse-trained tasks that are further from convergence. To estimate the excess risks, we develop an efficient and accurate method with Taylor approximation. Theoretically, we show that our proposed algorithm achieves convergence guarantees and Pareto stationarity. Empirically, we evaluate our algorithm on various MTL benchmarks and demonstrate its superior performance over existing methods in the presence of label noise. Our code is available at \href{https://github.com/yifei-he/ExcessMTL}{https://github.com/yifei-he/ExcessMTL}. \looseness=-1
\end{abstract}

\section{Introduction}
% Multi-task learning (MTL) aims to train a single model to perform multiple related tasks~\citep{caruana1997multitask}. Due to the nature of learning multiple tasks simultaneously, the problem is often tackled by reducing the multi-objective function into a scalar one via a convex combination. To determine the combination weight, one popular method is adaptive weight updating, which dynamically adjusts the weight for each task according to the task-specific loss~\citep{sagawadistributionally,liu2019end,chen2018gradnorm}. The intuition is that difficult tasks should be emphasized during training. However, these methods are not robust to \textit{label noise}, as the high loss may be due to label noise rather than the actual difficulty in the task. For instance, if a task has randomly assigned labels, its loss will be high, but no useful information can be learned from it. Moreover, the presence of such a noisy task can negatively impact the performance of all tasks trained jointly, leading to suboptimal performance across the board.

Multi-task learning (MTL) aims to train a single model to perform multiple related tasks~\citep{caruana1997multitask}. Due to the nature of learning multiple tasks simultaneously, the problem is often tackled by aggregating multiple objectives into a scalar one via a convex combination. Despite various efforts to achieve more balanced training, a crucial aspect that often remains overlooked is robustness to \textit{label noise}. Label noise is ubiquitous in real-world MTL problems as the tasks are drawn from diverse sources, introducing variations in data quality~\citep{Hsieh_Tseng_2021,Burgert_2022}. The presence of label noise can negatively impact the performance of all tasks trained jointly, leading to sub-optimal performance across the board. Addressing this issue is essential for the robust and reliable deployment of MTL models in real-world scenarios. 

In this work, we study label noise in MTL by delving into a practical scenario where \textit{one or more tasks are contaminated by label noise} due to the heterogeneity of data collection processes. 
Under label noise, scalarization (static weighted combination) is not robust because it overlooks the dataset quality. In fact, scalarization is prone to overfitting to a subset of tasks so that it cannot achieve a balanced solution among tasks, especially for under-parametrized models~\citep{hu2023revisiting}. Similarly, existing adaptive weight updating methods are vulnerable to label noise. These methods aim at prioritizing difficult tasks during training, and the difficulty is typically measured by the magnitude of each task loss~\citep{chen2018gradnorm,liu2019end,sagawadistributionally,liutowards}. Namely, they assign higher weights to the tasks with higher losses. However, the high loss may stem from label noise rather than insufficient training. For instance, if a task has noise in labels, its loss will be high, yet it provides no informative signal for the learning process. 
% \han{actual difficulty is a synonym for Bayes optimal error. I would say close to convergence etc}.
 % In fact, scalarization is prone to overfitting to a subset of tasks and cannot achieve Pareto optimal solutions that strike the balance among them, especially for underparametrized models~\citep{hu2023revisiting}. 
% Gradient balancing~\citep{yu2020gradient,liu2021conflict,liutowards,nashmtl} aims at reducing gradient conflict among tasks. As long as the gradient from the noisy task is considered, the optimization for all tasks will be damped or even misled.

\begin{figure}[t!]
    \centering
    \includegraphics[width=0.75\linewidth]{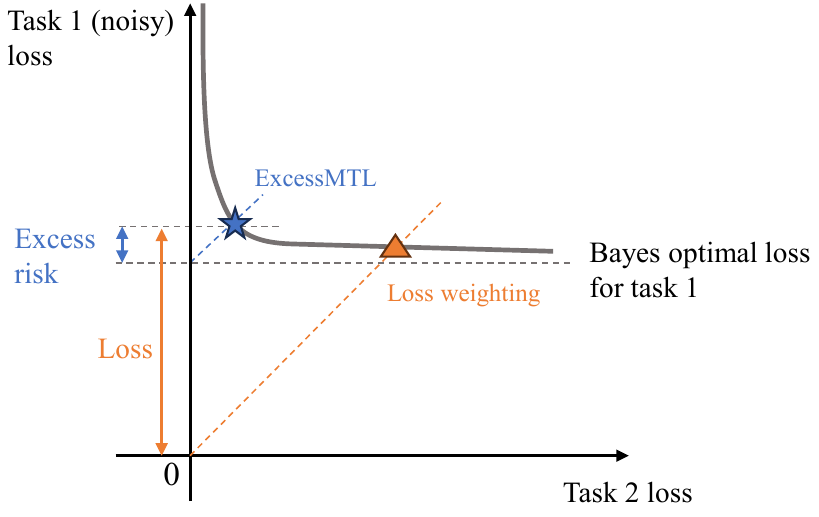}
    \caption{Conceptual comparison between ExcessMTL and loss weighting methods. The figure shows a two-task MTL setting, where Task 1 contains label noise, while Task 2 does not. Thus, the Bayes optimal loss (dashed line) for Task 1 is non-zero. The curve represents the Pareto front, i.e., all points on the curve are Pareto optimal. Loss weighting methods aim to find the solution with equal losses for two tasks, severely sacrificing the performance of Task 2. On the other hand, ExcessMTL finds the solution with equal excess risks, striking a better balance between the two tasks.}
    \label{fig:er_loss}
    \vspace{-0.3cm}
\end{figure}

To address this challenge, we propose Multi-Task Learning with Excess Risks (ExcessMTL), which is robust to label noise while retaining the benefit of prioritizing worse-trained tasks. It dynamically adjusts the task weights based on their distance to convergence. Specifically, we define the distance to be excess risks, which measure the gap of loss between the current model and the optimal model in the hypothesis class. In the presence of different noise levels among tasks, the converged task-specific losses are likely to differ and are not guaranteed to be low. On the other hand, with proper training, excess risks for all tasks will approach 0. Thus, excess risks provide the true improvement ceiling achievable through model training or refinement, making it naturally robust to label noise by definition. The advantage of excess risks over losses as a difficulty measure is demonstrated in \Cref{fig:er_loss}. Since the optimal loss is usually intractable to compute, we propose an efficient method to estimate excess risks via Taylor approximation. 

At a high level, ExcessMTL iteratively executes the following steps until convergence: i) estimate excess risk for each task, ii) update task weights based on their respective excess risks, iii) perform a gradient update using the weighted sum of losses. The ability to identify the convergence distance leads to robustness in the presence of label noise. Even if one or multiple tasks are highly corrupted by label noise, the overall performance will not be compromised.

Theoretically, we derive convergence guarantees for our algorithm and establish connections with multi-objective optimization, proving that the solutions of ExcessMTL are Pareto stationary. Empirically, we evaluate our method on various MTL benchmarks, showing that it outperforms existing adaptive weighting methods in the presence of label noise, even in cases of extreme noise in one or a few tasks. Our results highlight the robustness of excess risk-based weighting methods, especially when label noise is a concern. Beyond MTL, our insights on the connection between excess risks and convergence distance can potentially inspire more robust algorithmic design for applications using loss weighting methods.

 % including MultiMNIST~\citep{sabour2017dynamic}, CelebA~\citep{liu2015deep}, SARCOS~\citep{vijayakumar2000locally} and NYUv2~\citep{silberman2012indoor}

% To summarize, our contributions are
% \begin{itemize}
%     \item Propose a novel MTL algorithm robust to label noise.
%     \item Develop efficient and accurate methods to estimate the excess risks.
%     \item Demonstrate the convergence guarantee and Pareto optimality analysis of the proposed algorithm.
% \end{itemize}

\section{Preliminaries}
% We first go over the essential background on excess risks, multi-task learning and multi-objective optimization. 
\subsection{Excess Risks}
% We refer to the input and output space as $\cX,\cY$, respectively, and we use $X,Y$ to represent random variables that take values in $\cX,\cY$. 
Consider predicting the label $y\in\cY$ from the input data $x\in\cX$. Let the data be drawn from some distribution $P$. Given the per-sample loss function $\mathcal{L}$, the risk (or expected loss) of a model $\theta$ from the model family $\Theta$ is given by
\begin{align*}
    \ell(\theta)=\E_{(x,y)\sim P}[\mathcal{L}(\theta;(x,y))].
\end{align*}
The risk can be further decomposed as follows
\begin{align*}
    \ell(\theta)=\underbrace{\underbrace{\ell(\theta)-\ell(\theta_\Theta^*)}_\text{Estimation error}+\underbrace{\ell(\theta_\Theta^*)-\ell(\theta^*)}_\text{Approximation error}}_{\text{Excess risk} (\cE)}+\underbrace{\ell(\theta^*)}_\text{Bayes error},
    % \theta_\Theta=\argmin_{\theta\in\Theta}\E_{(x,y)\sim P}[\ell(\theta;(x,y))]\nonumber,\\ 
    % \theta^*=\argmin_\theta\E_{(x,y)\sim P}[\ell(\theta;(x,y))] \nonumber
\end{align*}
where $\theta_\Theta^*=\argmin_{\theta\in\Theta}\E_{(x,y)\sim P}[\mathcal{L}(\theta;(x,y))]\nonumber$ is the optimal model in the model family $\Theta$ and $\theta^*=\argmin_\theta\E_{(x,y)\sim P}[\mathcal{L}(\theta;(x,y))]$ is the optimal model in any model family. The combination of estimation error and approximation error is the excess risk, denoted as $\cE$. When the function class $\cF$ is expressive enough, the approximation error approaches 0. The Bayes error is irreducible due to the stochasticity in the data generating process (e.g., label noise). For instance, if the data generating process is non-deterministic, one data point can have non-zero probability of belonging to multiple classes. This is an inherent property of a dataset, where high label noise leads to high Bayes error. On the other hand, the excess risk captures the difference between the risks of the model and the Bayes optimal model, effectively removing the influence of label noise. Therefore, it can be viewed as a measure of distance to optimality. \looseness=-1

% The risk can be further decomposed as follows
% \begin{align}
%     \varepsilon(\theta)=\underbrace{\varepsilon(\theta)-\varepsilon(\theta^*)}_\text{Excess risk}+\underbrace{\varepsilon(\theta^*)}_\text{Bayes error},
% \end{align}
% \han{strictly speaking $\theta^*$ is not the Bayes optimal predictor, but the best predictor in the model class. The Bayes optimal predictor is the best predictor in the class of all predictors, which does not necessarily belong to the model class.}
% where $\theta^*$ is the Bayes optimal model, i.e., $\theta^*=\argmin_\theta\E_{(x,y)\sim P}[\ell(\theta;(x,y))]$. The Bayes error is the smallest achievable prediction error by any model, which is irreducible due to the stochasticity in the data generating process. For instance, if the data generating process is non-deterministic, one data point can have non-zero probability of belonging to multiple classes. This is an inherent property of a dataset, where high label noise leads to high Bayes error. On the other hand, the excess risk captures the difference between the risk of the model and the Bayes optimal model, so it can be viewed as a measure of distance to optimality. 
% It can be reduced to 0 as the sample size goes to infinity.

The above decomposition of risk highlights that it is not an appropriate criterion for evaluating model performance because it incorporates the irreducible error, which heavily depends on the noise level in the dataset. Since we typically do not have control over the dataset quality, excess risk is a more robust measure of model performance, particularly in the presence of label noise. It allows us to focus on the part of risk that can be improved through model learning and optimization, making it a reliable metric for assessing the effectiveness of models.

\subsection{Multi-Task Learning}
In MTL, $m \geq 2$ tasks are given. We use $\alpha_i$ to represent the weight of the $i$-th task and $\Delta_m$ to denote the $(m-1)$-dimensional probability simplex. We study the setting of hard parameter sharing, where a subset of model parameters ($\theta_{sh}$) is shared across all tasks, while other parameters ($\theta_i$) are task-specific. The goal of MTL is to find the parameter $\theta_{sh}$ and $\theta_{i}$ that minimizes a convex combination of all task-specific losses ($\ell_i$)
\vspace{-0.3cm}
% Each task is associated with a task-specific loss $\ell_i$. At training step $t$, for the gradient and Hessian matrix of the model parameter $\theta^{(t)}$, we short hand $\nabla_{\theta^{(t)}}\ell_i(\theta^{(t)})$ as $g_i^{(t)}$ and $H_i(\theta^{(t)})$ as $H_i^{(t)}$.
\begin{align}\label{mtl-obj}
    \min_{\theta_{sh}, \theta_1,\cdots,\theta_m} \sum_{i=1}^m \alpha_i\ell_i(\theta_{sh}, \theta_i),
\end{align}
where $\alpha_i \in \Delta_m$. The weights can be either static~\citep{liu2021conflict,yu2020gradient} or dynamically computed~\citep{chen2018gradnorm,liu2019end,liutowards,nashmtl}.

Despite the wide usage of the weighted combination scheme, it is difficult to define optimality under the MTL setting because a model may work well on some tasks, but perform poorly on others. For more rigorous optimality analysis, MTL can be formulated as a multi-objective optimization problem~\citep{sener2018multi,zhou2022convergence}, where two models can be compared by Pareto dominance.
\begin{definition}[Pareto dominance]
    Let $L(\theta)=\{\ell_i(\theta):i\in[m]\}$ be a set of loss functions. For two parameter vectors $\theta_1$ and $\theta_2$, if $\ell_i(\theta_1)\leq\ell_i(\theta_2)$ for all $i\in[m]$ and $L(\theta_1)\neq L(\theta_2)$, we say that $\theta_1$ Pareto dominates $\theta_2$ with the notation $\theta_1 \prec \theta_2$.
\end{definition}
The goal of multi-task learning as multi-objective optimization is to achieve Pareto optimality.
\begin{definition}[Pareto optimal]\label{def:pareto-optimal}
A parameter vector $\theta^*$ is Pareto optimal if there exists no parameter $\theta$ such that $\theta \prec \theta^*$.
\end{definition}
There may exist multiple Pareto optimal solutions and they consist the Pareto set. A weaker condition is called Pareto stationary and all Pareto optimal points are Pareto stationary.
\begin{definition}[Pareto stationary]
    A parameter $\theta$ is called Pareto stationary if there exists $\alpha \in \Delta_{m}$ such that $\sum_{i=1}^m \alpha_i\nabla_\theta \ell_i(\theta)=0$.
\end{definition}

\section{Multi-Task Learning with Excess Risks}

We now introduce our main algorithm. In \Cref{motivation}, we begin by identifying a key limitation of previous MTL algorithms and outlining our objectives. In \Cref{algo}, we provide a detailed description of the algorithm. In \Cref{conceptual_comparison}, we make a conceptual comparison between our proposed algorithm with existing task weighting methods. Finally, in \Cref{theory}, we present a theoretical analysis of the convergence guarantee and Pareto stationarity of our algorithm. \looseness=-1

\subsection{Motivations and Objectives} \label{motivation}

Prior works have demonstrated that effective multi-task learning requires task balancing, i.e., tasks that the current model performs poorly on should be assigned higher weights during training~\citep{guo2018dynamic,kendall2018multi}. One popular method is loss balancing, which assigns high weights to the tasks with high losses, in the hope of prioritizing difficult tasks and making the training more balanced. However, we argue that task-specific loss is not a good criterion for task difficulty, as it not only considers model training, but is also subject to the quality of the dataset, which we may not have prior knowledge about. For instance, the high loss may not necessarily stem from insufficient training, but from high label noise. In that case, assigning high weights to the noisy tasks can hurt the multi-task performance across the board.
% \han{using ``is subject to'' is better}

To address this problem, we propose to use excess risks to weigh the tasks because it measures the true performance gap that can be closed by model training. We hope to assign high weights to tasks with high excess risks such that all tasks converge at similar rates. To achieve this goal, we solve the min-max problem
\begin{align} \label{og_obj}
    \min_{\theta_{sh},\theta_1,\cdots\theta_m} \max_{i\in[m]}~\cE_i(\theta_{sh}, \theta_i),
\end{align}
where $\cE_i$ is the excess risk of the $i$-th task. Note that the formulation has the same form as Tchebycheff scalarization function, whose solution is able to fully explore the Pareto front of multiple objectives~\cite{bowman1976relationship,steuer1983interactive}, a property lacking in linear scalarization~\citep{hu2023revisiting}.
% \han{This is the first place where we introduce the notation of excess risk, might be better to put it upfront.}
%Solving this problem requires selecting the task with the highest excess risk in each iteration, which can introduce instability during training. 
% Following similar ideas from follow-the-regularized-leader (FTRL) with entropic regularization~\citep{kalai2005efficient}, we smooth out the hard choice of a single task with a weight vector $\alpha\in\Delta_m$ and apply entropy regularization to discourage a one-hot solution. 

% where $\eta_\alpha$ controls the strength of the regularization. This formulation improves the stability and provides a convergence guarantee as shown in \Cref{theory}. 
% \han{I don't think there is a difference between (5) and (4), as the maximum over (5) exactly reduces to (4). I understand the argument you'd like to convey here, but without the derivation, it's hard to express the idea. Maybe you can mention that our strategy of smoothing out is similar in spirit of Follow-the-Leader to Follow-the-Regularized-Leader with entropic regularization.}

\begin{algorithm}[t!]
\caption{ExcessMTL}\label{algo:main}
\begin{algorithmic} %[1]
\STATE \textbf{Input:} Step size $\eta_\alpha$, $\eta_\theta$, number of total tasks $m$
\STATE Initialize $\theta_{sh}^{(1)}$ and $\theta_i^{(1)}$ for all $i\in[m]$, $\alpha^{(1)}=[1/m, \cdots, 1/m]$
\FOR{$t=1,2 \cdots$}
    \FOR{$i=1, \cdots, m$}
        % \STATE Compute the task-specific gradient \\
        \STATE Compute gradient $g_i^{(t)}=\nabla_{\theta_{sh}}\ell_i^{(t)}(\theta_{sh}, \theta_i)$\\
        \STATE Compute excess risks with Eq. \ref{eq:excess} and Eq. \ref{eq:hessian}\\
        \STATE $\hat{\cE}_i^{(t)}={g_i^{(t)}}^\top {\text{diag}  \left( \sum_{\tau=1}^t g_i^{(\tau)} {g_i^{(\tau)}}^\top  \right)^{-1/2}} g_i^{(t)}$ \\
        % \STATE Update task weights \\
        \STATE Update weights $\alpha_i^{(t+1)}=\alpha_i^{(t)}\exp \left( \eta_\alpha \hat{\cE}_i^{(t)} \right)$\\
        % $\alpha_i^{(t+1)}=\alpha_i^{(t)}\exp \left( \eta_\alpha {g_i^{(t)}}^\top {\text{diag}  \left( \sum_{\tau=1}^t g_i^{(\tau)} {g_i^{(\tau)}}^\top  \right)^{-1/2}} g_i^{(t)}   \right)$
        \STATE Update task-specific parameters \\
        \STATE $\theta^{(t+1)}_i  \leftarrow \theta^{(t)}_i - \eta_\theta \nabla_{\theta_i}\ell_i^{(t)}(\theta_{sh}, \theta_i)$ 
        % \hfill// Update task-specific parameters
    \ENDFOR
    \STATE Normalize $\alpha_i^{(t+1)}  \leftarrow \alpha_i^{(t+1)} / \sum_j \alpha_j^{(t+1)}$ for all $i$ %\in [m]$ 
    % \hfill// Normalization
    \STATE Update shared parameters
    \STATE $\theta^{(t+1)}_{sh}  \leftarrow \theta^{(t)}_{sh} - \eta_\theta\sum_i \alpha_i^{(t+1)}\nabla_{\theta_{sh}}\ell_i^{(t)}(\theta_{sh}, \theta_i)$  
    % \hfill// Update shared parameters
\ENDFOR
\end{algorithmic}
\end{algorithm}

\subsection{Algorithm} \label{algo}
We present our algorithm in \Cref{algo:main}, which consists of three main components: excess risk estimation, multiplicative weight update, and scale processing.

\textbf{Excess risks estimation.}~For the ease of presentation, we overload the expression $\theta_i$ as a combination of $\theta_{sh}$ and $\theta_i$ for task $i$. In the computation of excess risks, the Bayes optimal loss is generally intractable to compute exactly, so we propose to use a local approximation instead. Specifically, we use the second-order Taylor expansion of the task-specific loss $\ell_i$ at the current parameter $\theta_i^{(t)}$ 
% \han{The last remainder term is not correct -- should be $O(\|\theta_i - \theta_i^{(t)}\|_2^3)$. Also, since the expansion contains the remainder term, it should be an exact $=$ rather than $\approx$.}
\begin{align}
    % \small
    % \ell_i(\theta_i) &= \ell_i(\theta_i^{(t)})+(\theta_i-\theta_i^{(t)})^\top g_i^{(t)} \nonumber \\
    % &+ \frac{1}{2} (\theta_i-\theta_i^{(t)})^\top H_i^{(t)} (\theta_i-\theta_i^{(t)})+\cO(\|\theta_i - \theta_i^{(t)}\|_2^3), \label{eq:taylor}
    \ell_i&(\theta) = \ell_i(\theta_i^{(t)})+(\theta-\theta_i^{(t)})^\top g_i^{(t)} \nonumber \\
    &+ \frac{1}{2} (\theta-\theta_i^{(t)})^\top H_i^{(t)} (\theta-\theta_i^{(t)})+\cO(\|\theta - \theta_i^{(t)}\|_2^3), \label{eq:taylor}
\end{align}        
where $g_i^{(t)}$ is the gradient and $H_i^{(t)}$ is the Hessian matrix of $\ell_i$ at $\theta_i^{(t)}$. Plugging the locally optimal parameter $\theta_i^*$ in Eq. \ref{eq:taylor}, we can estimate the excess risk as 
% \han{In Eq. (8), the sign before the quadratic term should be $-$ instead of $+$.}
\begin{align}
    &\cE_i(\theta_i^{(t)}) \approx \ell_i(\theta_i^{(t)}) - \ell_i(\theta_i^*) \nonumber\\ 
    \approx &(\theta_i^{(t)}-\theta_i^*)^\top g_i^{(t)} - \frac{1}{2} (\theta_i^{(t)}-\theta_i^*)^\top H_i^{(t)} (\theta_i^{(t)}-\theta_i^*). \label{eq:second}
\end{align}
% \han{Since $\theta_i^*$ is only locally optimal, the first equation should also be approximate.}
To obtain the difference between $\theta_i^{(t)}$ and $\theta_i^*$, we use the fact that $\theta_i^*$ is locally optimal,
\begin{align}
    \nabla_{\theta_i^*} \ell_i(\theta_i^*) &\approx g_i^{(t)}+H_i^{(t)}(\theta_i^*-\theta_i^{(t)})= 0 \nonumber\\
    \implies\theta_i^{(t)}-\theta_i^*&\approx{H_i^{(t)}}^{-1}g_i^{(t)}. \label{eq:distance}
\end{align}
Plugging Eq.~\ref{eq:distance} into Eq.~\ref{eq:second} and assuming the second-order partial derivative is continuous, we have
 % \han{why do we need the continuity here?} Yifei: there is a part in the proof about subtracting hessian inverse from its transpose. Continuous derivative makes the result more concise as we can treat hessian inverse and its transpose as the same thing.
\begin{align}
    \cE_i(\theta_i^{(t)}) &\approx \frac{1}{2}{g_i^{(t)}}^\top {H_i^{(t)}}^{-1} g_i^{(t)}. \label{eq:excess}
\end{align}
% \han{Is this true? The Fisher information matrix indeed does have this form, but the loss function needs to be the log-likelihood function.} 
The factor $1/2$ can be dropped for simplicity. However, the computation of the Hessian matrix is generally intractable, so we use the diagonal approximation of empirical Fisher~\citep{Amari1998NaturalGW} to estimate it, which computes a diagonal matrix through the accumulation of outer products of historical gradients. This approach has shown to be useful in various practical settings~\citep{duchi2011adaptive,adam}. Specifically, let $g_i^{(\tau)}$ be the gradient of the $i$-th task with respect to the model parameter $\theta_i^{(\tau)}$ at training step $\tau$, the approximate Hessian at training step $t$ for the $i$-th task is
\begin{align}
    {H^{(t)}_i} \approx \text{diag}  \left( \sum_{\tau=1}^t g_i^{(\tau)} {g_i^{(\tau)}}^\top  \right)^{1/2}, \label{eq:hessian}
\end{align}
where $\text{diag}(\cdot)$ is a diagonal matrix. The estimation is efficient as the computational complexity is $\cO(d)$, where $d$ is the dimension of parameters.

% \han{In the literature I think it's called multiplicative weight update rather than exponential weight update. Also, the logic transition here seems a bit abrupt -- what's the justification or motivation to use the multiplicative weight update algorithm? You can at least provide the derivation by following the same derivation in online learning by using entropic regularization. See my detailed comment in the Slack channel.}
\textbf{Multiplicative weight update.}~ After estimating the excess risks, we update the task weights accordingly. As the gradients contain stochastic factors during the training, this online learning process calls for the stability of $\alpha$ for algorithmic convergence~\citep{hazan2019introduction}. The one-hot solution in \Cref{main-obj} suffers from fluctuation and cannot be directly applied. Following the framework of online mirror descent with entropy regularization that formulates KL divergence as Bregman divergence~\citep{hazan2019introduction}, we smooth out the hard choice of a single task with a weight vector $\alpha\in\Delta_m$, so problem \ref{og_obj} can be reformulated as
\vspace{-0.1cm}
\begin{align}\label{main-obj}
    \min_{\theta_{sh},\theta_1,\cdots,\theta_m} \max_{\alpha\in\Delta_m} \sum_{i=1}^m \alpha_i\cE_i(\theta_{sh}, \theta_i).
\end{align}
Here we find the task weights $\alpha^{(t)}$ by
\begin{align*}
    \alpha^{(t+1)} =& \argmax_{\alpha\in\Delta_m} \sum_{i=1}^m \alpha_i\cE_i(\theta_{sh}^{(t)}, \theta_i^{(t)})
    -\frac{1}{\eta_\alpha} \text{KL} (\alpha\|\alpha^{(t)}),
\end{align*}
% To compute this maximum under the simplex constraint, we can introduce the Lagrange multipliers. Specifically, we can transfer the above constrained maximization problem into
% \begin{align}
%     \alpha^{(t)} = \argmax_{\alpha} \sum_{i=1}^m \alpha_i\cE_i(\theta_{sh}^{(t)}, \theta_i^{(t)})-\frac{1}{\eta_\alpha} \sum_{i=1}^m\alpha_i \log \alpha_i - \lambda\cdot(\langle \alpha, \mathbf{1}_m \rangle - 1),
% \end{align}
% where $\mathbf{1}_{m}$ is the $m$-dimension vector with all 1s. Then, for each $\alpha_i$, the partial derivative is
% \begin{align}
%     \cE_i(\theta_{sh}^{(t)}, \theta_i^{(t)})-\frac{1}{\eta_\alpha}(1+\log\alpha_i)-\lambda.
% \end{align}
% Setting it to be zero, we obtain that 
% \begin{align}
%     \alpha_i = \exp\left(1-\eta_\alpha\lambda+\eta_\alpha\cE_i(\theta_{sh}^{(t)}, \theta_i^{(t)})\right).
% \end{align}
% Note that to ensure $\alpha$ is a probability vector, $\lambda$ has a unique solution and our update rule is
% perform one step of multiplicative weight update on the task weights with the step size $\eta_\alpha$,
which can easily be proven as equivalent to
\begin{align}
    \alpha_i^{(t+1)}= \frac{\alpha_i^{(t)}\exp \left( \eta_\alpha \cE_i(\theta_{sh}^{(t)},\theta_i^{(t)})  \right)}{\sum_{j=1}^m \alpha_j^{(t)}\exp \left( \eta_\alpha \cE_j(\theta_{sh}^{(t)},\theta_j^{(t)})  \right)},
\end{align}
where $\eta_\alpha$ is the step size for weight update. Note that this update can be viewed as an exponentiated gradient~\citep{kivinen1997exponentiated} step on the convex combination of excess risks. 
% To map the task weights back to the simplex $\Delta_m$, we perform normalization in each iteration. 
Based on the updated task weights, we compute the loss by taking a convex combination of each task-specific loss and backpropagating the gradient. Multiplicative weight update provides stability in training and we show the convergence guarantees in \Cref{theory}. 
% \han{If you do the derivation using entropic regularization, the normalization step will naturally pop up as the optimal solution to $\alpha$ rather than manually performing the normalization so that it becomes a probability vector. }

\textbf{Scale processing.}~
Similar to loss and gradient, the excess risk is also sensitive to the scale of tasks. The issue can manifest in two scenarios. The first one is when the type of loss is the same across all tasks, but the input data for each task varies in magnitude. A simple solution is to standardize all input data such that they have zero mean and unit variance. The second case is when the tasks employ different losses. For instance, the cosine loss has a maximum value of 1, while the squared loss can be unbounded. To ensure a fair comparison among tasks with varying losses, we propose to compute the relative excess risks. We adopt a similar method from~\citet{chen2018gradnorm} to normalize excess risks by dividing the current value by the initial value, ensuring a range from 0 to 1. 
 % (e.g., regression tasks with squared loss)

% However, this approach can be unstable since the initial excess risk heavily depends on the initialization of the network, which can be arbitrarily poor. To address this, we propose a warm-up strategy, where we use a small step size to update the weights in the first several epochs and compute the average excess risks over those epochs as an estimation of the initial excess risk.

% Note that our proposal of relative excess risks is different from previous weighting schemes to deal with the scale problem.~\citet{liutowards} proposes to learn a set of weights such that the weighted losses have comparable scale. However, without preprecessing of the losses, the weights will be concentrated on the tasks with losses that have inherently small scale regardless of the actual task difficulty. 

\subsection{Conceptual Comparison} \label{conceptual_comparison}

In this section, we discuss the relationships and distinctions between our approach and prior task-balancing methods conceptually. We demonstrate the limitations of previous methods, especially in the presence of label noise. Empirical verification of the arguments is presented in \Cref{emp_comp}.

\textbf{GradNorm}~\citep{chen2018gradnorm}~enforces similar gradient norms and training rates across all tasks. The training rate is defined as the ratio of the current and the initial loss. It tends to favor tasks with small gradient norms, overlooking inherent scale differences in task gradients. In the face of label noise, noisy tasks inevitably exhibit a slow training rate due to their high losses. However, GradNorm treats this as insufficient training and assigns high weights to the noisy tasks. Additionally, the task weights require additional gradient updates as they are learnable parameters in GradNorm, whereas our algorithm does not. 

% \textbf{Multiple Gradient Descent Algorithm (MGDA)} has been studied for a long time in the literature~\citep{mukai,FliegeSvaiter00,desideri2012multiple}. 
\textbf{MGDA}~\citep{sener2018multi}~formulates multi-task learning as multi-objective optimization. They solve it by using the classical multiple gradient descent algorithm~\citep{mukai,FliegeSvaiter00,desideri2012multiple}, which finds the minimum norm within the convex hull of task gradients. It also tends to favor tasks with small gradient magnitudes due to the nature of the Frank-Wofle algorithm it employs. Consider two tasks with gradients $\|g_1\|_2>\|g_2\|_2$ where the projection of $g_1$ on $g_2$ has a larger norm than $g_2$. In this case, MGDA will concentrate all weights on task 2. This observation holds in general beyond the two-task example. As the loss landscape of a noisy task tends to be flatter, its gradient magnitude will be smaller, thus favored by MGDA. Moreover, the Frank-Wofle algorithm requires pairwise computation among all tasks, which is prohibitive under large number of tasks.

\textbf{GroupDRO} \citep{sagawadistributionally}~is initially designed to address the problem of subpopulation shift and has since been extended to multi-task learning~\citep{michel2021balancing}. It aims to optimize the worst task loss by assigning high weights to tasks with high losses. However, in the presence of label noise, tasks with high noise levels will exhibit persistently high losses, which leads GroupDRO to assign excessive weight to those tasks, thereby ignoring the other tasks and causing an overall performance decline.

\textbf{IMTL} \citep{liutowards} achieves impartial learning by requiring the gradient update to have equal projection on each task. In a two-task scenario, this is equivalent to finding the angle bisector for the task-specific gradients, meaning that regardless of the gradient magnitude, each task exerts an equal influence on determining the final gradient direction. However, when confronted with label noise, the noisy gradient can substantially distort the final gradient direction, damping or even misguiding the training.

% \textbf{MOML} \citep{ye2021multiobjective} formulates the problem as a multi-objective meta learning problem. It assign weights using MGDA on a clean validation set. However, even if a clean validation set is available, MOML is \textit{not} robust to label noise because it fails to account for noise in the training set, still leading to considerable weights for noisy tasks.

\subsection{Theoretical Analysis} \label{theory}

In this section, we analyze ExcessMTL theoretically, and verify its soundness with three key properties: ExcessMTL i) converges at the same rate as MGDA and GroupDRO (\Cref{thm:convergence}), ii) converges to Pareto optimal solutions in convex cases (\Cref{co:pareto-optimal}), iii) converges to Pareto stationary solutions in non-convex cases (\Cref{thm:convergence_nonconvex}).

\Cref{algo:main} takes inspiration from the online mirror descent algorithm~\citep{nemirovskij1983problem}, where the update of parameter $\theta$ and weights $\alpha$ corresponds to online gradient descent and online exponentiated gradient respectively. With established results~\citep{nemirovski2009robust} in the online learning literature, we show that \Cref{algo:main} converges at the rate $\cO(1/\sqrt{t})$, same as MGDA and GroupDRO.
\begin{restatable}[Convergence]{theorem}{convergence}\label{thm:convergence}
    Suppose (i) each task-specific loss $\ell_i$ is L-Lipschitz, (ii) $\ell_i$ is convex on the model parameter $\theta$, (iii) $\ell_i$ bounded by $B_\ell$ and (iv) $\|\theta\|_2$ is bounded by $B_\theta$. At training step $t$, let $\bar{\theta}^{(1:t)}\coloneqq \frac{1}{t}\sum_{\tau=1}^t{\theta^{(\tau)}}$, then
    \begin{align}
        &\E\left[\sum_{i=1}^m \alpha_i^{(t)}\cE_i(\bar{\theta}^{(1:t)})\right] - \min_{\theta} \max_{\alpha\in\Delta_m} \sum_{i=1}^m \alpha_i\cE_i(\theta) \nonumber\\
        \leq &2m\sqrt{\frac{10(B_\theta^2 L^2 + B_\ell^2 \log m)}{t}},
    \end{align}
    where $m$ is the number of tasks.
\end{restatable}
With the convergence analysis, we further deduce that the solution of ExcessMTL is Pareto optimal in the convex setting. \looseness=-1
 % \han{Actually there is a gap: the guarantee is for averge model parameters but our algorithm is a last-iteration algorithm instead of the average one.} 

% \begin{restatable}{corollary}{pareto}\label{co:pareto-optimal}
%     Under the same conditions as Theorem \ref{thm:convergence}, there exists a set of weights $\lambda\in\Delta_m$ such that
%     \begin{align}
%         &\E\left[\sum_{i=1}^m \lambda_i\cE_i(\bar{\theta}^{(1:t)})\right] - \sum_{i=1}^m \lambda_i\cE_i(\theta^*) \nonumber \\
%         \leq &2m\sqrt{\frac{10(B_\theta^2 L^2 + B_\ell^2 \log m)}{t}},
%     \end{align}
%     where $\theta^*$ is the Pareto-optimal solution, i.e., $\theta^*=\argmin_\theta \sum_{i=1}^m \lambda_i\cE_i(\theta)$.
% \end{restatable}
\begin{restatable}[Pareto Optimality]{corollary}{pareto}\label{co:pareto-optimal}
    Under the same conditions as Theorem \ref{thm:convergence}, using the weights $\alpha$ output by \Cref{algo:main}, we have
    \begin{align}
        &\E\left[\sum_{i=1}^m \alpha_i\cE_i(\bar{\theta}^{(1:t)})\right] - \sum_{i=1}^m \alpha_i\cE_i(\theta^*) \nonumber \\
        \leq &2m\sqrt{\frac{10(B_\theta^2 L^2 + B_\ell^2 \log m)}{t}},
    \end{align}
    where $\theta^*$ is the Pareto optimal solution, i.e., $\theta^*=\argmin_\theta \sum_{i=1}^m \alpha_i\cE_i(\theta)$.
\end{restatable}
This condition is the same as \Cref{def:pareto-optimal} because the weights $\alpha$ are positive, so any Pareto improvement over $\theta^*$ would increase the sum. The proof is in \Cref{proof:pareto_opt}. For non-convex cases, we next provide a stationary analysis. 
% \han{The definition of $\lambda$ in the corollary above only suggests non-negative. If it's nonnegative, it's only Pareto-stationary, right?}

\begin{restatable}[Pareto Stationarity]{theorem}{paretostationary}\label{thm:convergence_nonconvex}
Suppose each task-specific loss $\ell_i$ is (i) L-Lipschitz (ii) G-Smooth and (iii) bounded by $B_\ell$. At training step $t$, using the weights $\alpha$ output by \Cref{algo:main}, we have
% there exists a set of weights $\lambda\in\Delta_m$ such that
    \begin{equation}
    \min_{k=1,\ldots,t}\mathbb{E}\left[\left\|\sum_{i=1}^m {\alpha_i^{(k)}} \nabla \ell_i (\theta^{(k)}) \right\|^2_2\right]\leq 6m\sigma \sqrt{\frac{B_{\ell} G }{t}} ,
    \end{equation}
    where $m$ is the number of tasks, and $\sigma$ corresponds to the assumption of a bounded approximation error of the subgradient and the gradient of the excess risk detailed in \Cref{as:grad} of \Cref{proof:pareto_opt}.
\end{restatable}
This shows that ExcessMTL converges to a Pareto stationary point with the rate $\cO(1/t^{1/4})$, matching the rate for single-objective SGD~\citep{drori2020complexity}. We also empirically show that the algorithm performs well under the non-convex setting in \Cref{exp}. The proof is in \Cref{proof:pareto_stationary}.

The results demonstrate the theoretical soundness of our method, as it provably converges to optimal solutions in multi-objective optimization. MGDA and GroupDRO also provably converge to Pareto optimal solutions in the convex setting, while GradNorm and IMTL lack this property, meaning that their solutions can be Pareto dominated.

\section{Experiments} \label{exp}
The experiments aim to investigate the following questions under different levels of noise injection: i) Does label noise significantly harm MTL performance? ii) Does ExcessMTL perform consistently with its theoretical properties? iii) In presence of label noise, does ExcessMTL maintain high overall performance? iv) If so, is this achieved by appropriately assigning weights to the noisy tasks? In the subsequent analysis, we provide affirmative answers to all the questions.

\subsection{Datasets}

\textbf{MultiMNIST~\citep{sabour2017dynamic}} is a multi-task version of the MNIST dataset. Two randomly selected MNIST images are put on the top-left and bottom-right corners respectively to construct a new image. Noise is injected into the bottom-right corner task.
% The task is to predict the two digits. 

% \textbf{CelebA~\citep{liu2015deep}} consists of face images annotated with 40 attributes, where each attribute presents a binary classification task and the dataset can be treated as a 40-way MTL problem. 

\textbf{Office-Home~\citep{venkateswara2017deep}} consists of four image classification tasks: artistic images, clip art, product images, and real-world images. It contains 15,500 images over 65 classes. Noise is injected into the product image classification task.

\textbf{NYUv2~\citep{silberman2012indoor}} consists of RGB-D indoor images. It contains 3 tasks: semantic segmentation, depth estimation, and surface normal prediction. Noise is injected into the semantic segmentation task.

\subsection{Noise Injection Scheme}
% In this section, we explain our noise injection approach, which is employed for different types of tasks. 
To replicate the real-world scenario where tasks come from heterogeneous sources with various quality, we inject label noise into one or more (but not all) tasks within the task batch. For classification problems, we adopt the common practice of introducing symmetric noise~\citep{kim2019nlnl}, which is generated by flipping the true label to any possible labels uniformly at random. For regression problems, we introduce additive Gaussian noise~\citep{husimple} with the variance equal to that of the original regression target. This approach ensures that the level of noise introduced is consistent with the statistical characteristics of the original data. We vary the noise level by changing the proportion of training data subject to the noise injection procedure, allowing us to measure the impact of increasing levels of noise on the performance of our algorithm. The test data remains clean. \looseness=-1
% More details can be found in Appendix \ref{supp:exp}.

% To control the actual noise level in each dataset, we introduce noise into the task(s) with the best single-task performance, as these tasks typically have the lowest inherent label noise. For NYUv2, we perform noise injection on each of the three tasks.

In the following sections, we refer to the tasks with noise injection as \textit{noisy tasks} and tasks without noise injection as \textit{clean tasks}. The proportion of noisy data in the noisy task is referred to as \textit{noise level}.

% \subsection{Implementation}
% We implement our algorithm using hard parameter sharing, where all tasks share a feature extractor and have task-specific heads. For feature extractors, we use a two-layer CNN for MultiMNIST, ResNet-18~\citep{he2016deep} for Office-Home, and SegNet~\citep{badrinarayanan2017segnet} for NYUv2. For task-specific heads, we use two-layer CNNs for NYUv2, and MLP for all other datasets. We standardize all datasets to ensure zero mean and unit variance, as excess risks are sensitive to the scale of tasks. More details are provided in Appendix \ref{supp:exp}.
% three-layer MLP for SARCOS

\subsection{Empirical Analysis and Comparison} \label{emp_comp}

In this section, we use MultiMNIST and Office-Home as illustrative examples to analyze the behavior of ExcessMTL in detail and demonstrate its advantage comparing with other task weighting methods.

\begin{figure}[t!]
    \centering
    \includegraphics[width=0.9\linewidth]{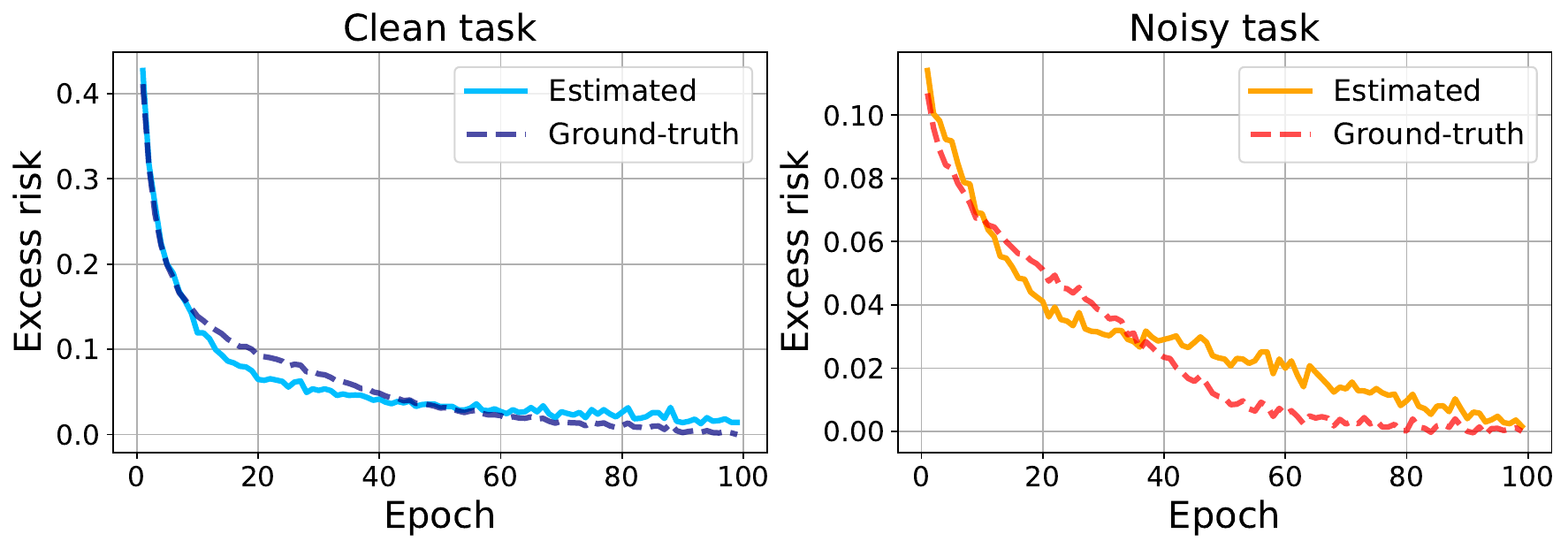}
    \vspace{-0.3cm}
    \caption{Excess risks on MultiMNIST with noise level 0.6. The estimated excess risk well matches the ground-truth pattern.}
    \label{fig:er}
    \vspace{-0.2cm}
\end{figure}

\textbf{Excess risk estimation.}~We present the estimated excess risks on MultiMNIST in Fig. \ref{fig:er}. Here, we use the difference between the current and converged loss as a proxy for the ground-truth excess risk. Initially, the noisy task shows lower excess risk due to higher Bayes optimal loss. As training proceeds, both tasks converge and excess risks approach 0. The estimation is up to a constant multiplier, so we scale the estimated value by a constant to align it with the ground-truth. Our estimated excess risk well matches the ground-truth pattern, validating its accuracy. 

\begin{figure}[t!]
    \centering
    \includegraphics[width=\linewidth]{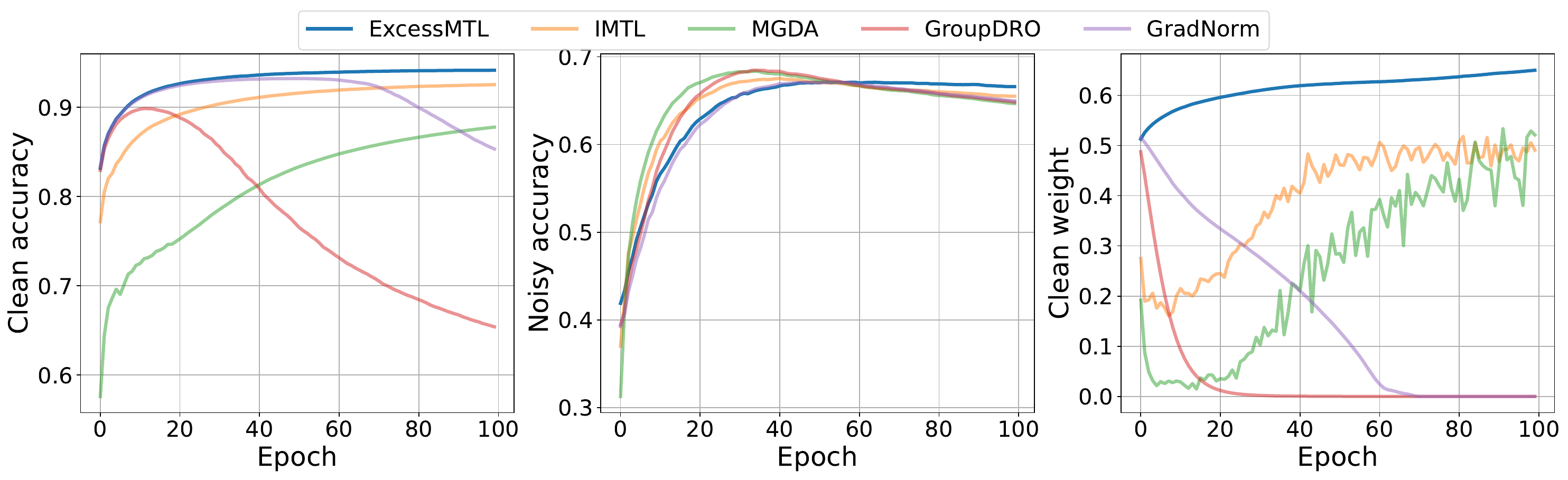}
    \vspace{-0.7cm}
    \caption{Weight and accuracy on the MultiMNIST dataset with a noise level of 0.8. ExcessMTL assigns most weight to the clean task so that the performance is least affected by the injected noise.}
    \label{fig:comparison}
    % \vspace{-0.2cm}
\end{figure}

\begin{figure}[t!]
    \centering
    \includegraphics[width=0.9\linewidth]{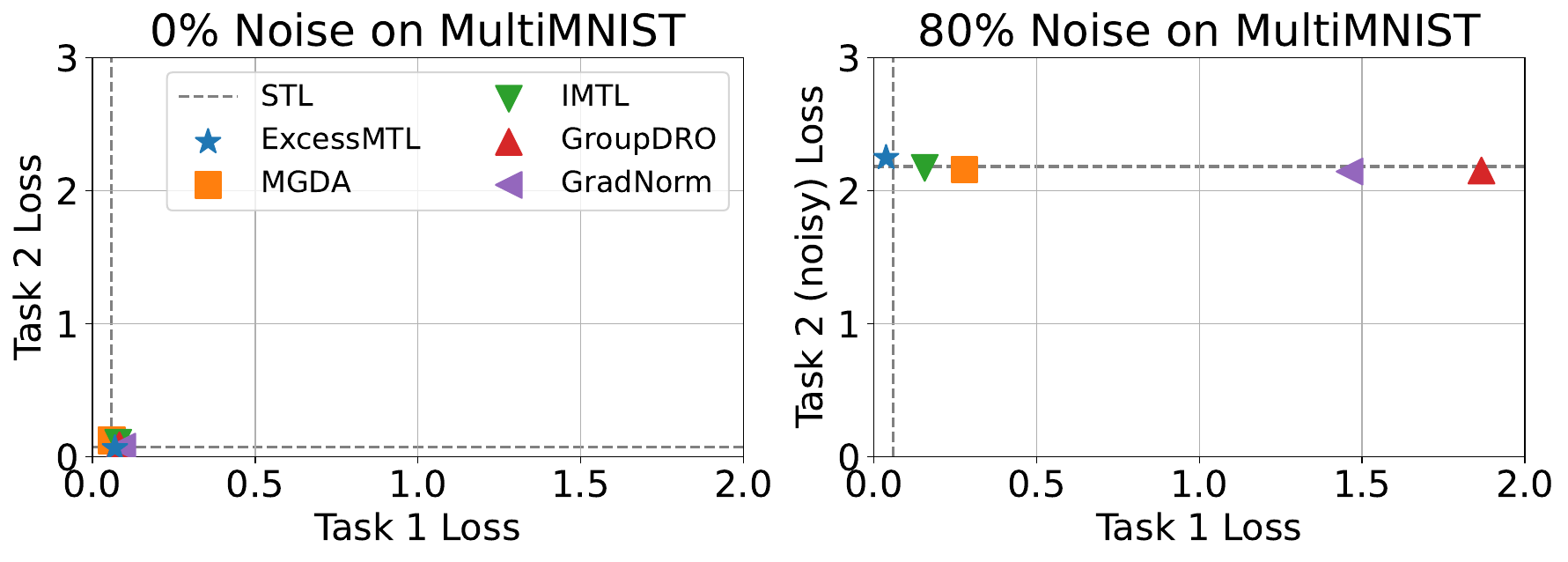}
    \vspace{-0.2cm}
    \caption{MultiMNIST loss profile (lower left better). The left plot has no noise injected, while in the right one, task 2 has 80\% noise. With no noise injected, all algorithms achieve ideal performance. However, with significant noise injected, only ExcessMTL retains performance close to Bayes optimal on both tasks.}
    \label{fig:pareto}
    \vspace{-0.2cm}
\end{figure}

\begin{figure*}[t!]
    \centering
    \includegraphics[width=0.75\linewidth]{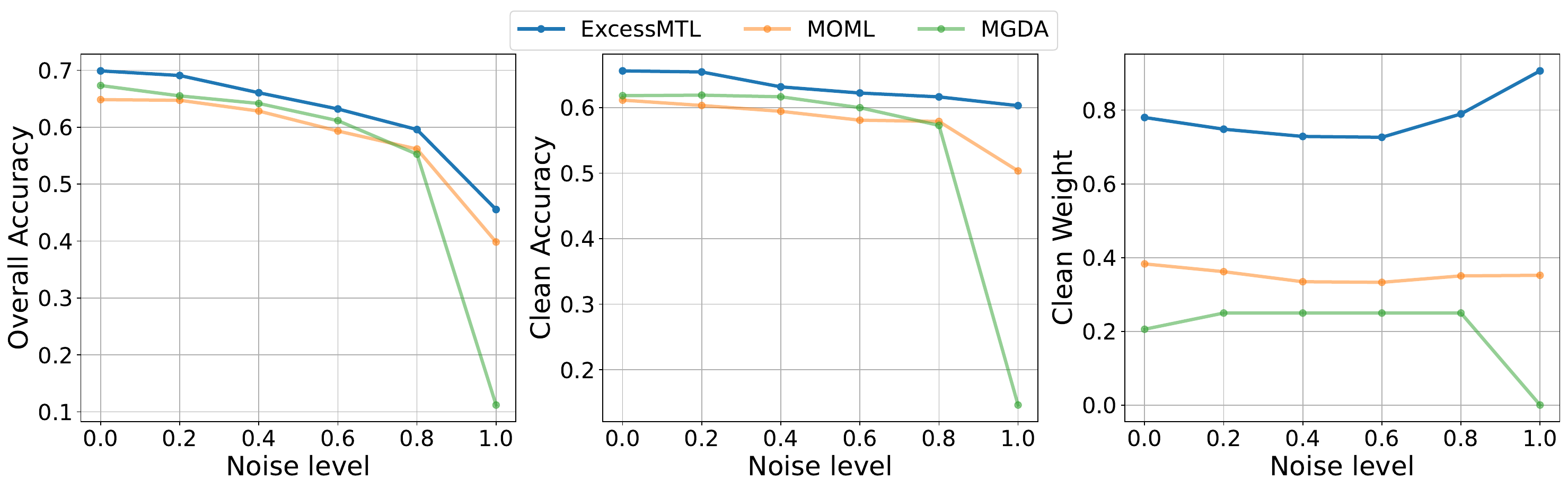}
    \vspace{-0.4cm}
    \caption{Comparison with MOML and MGDA. MGDA and MOML use the same method to select weights on training and validation set respectively. Despite more consistent weight assignment than MGDA, MOML fails when noise level is high, showing that a clean validation set does not alleviate the label noise issue. ExcessMTL ourperforms both baselines.}
    \label{fig:moml}
    % \vspace{-0.2cm}
\end{figure*}

\begin{figure*}[t!]
    \centering
    \includegraphics[width=0.75\linewidth]{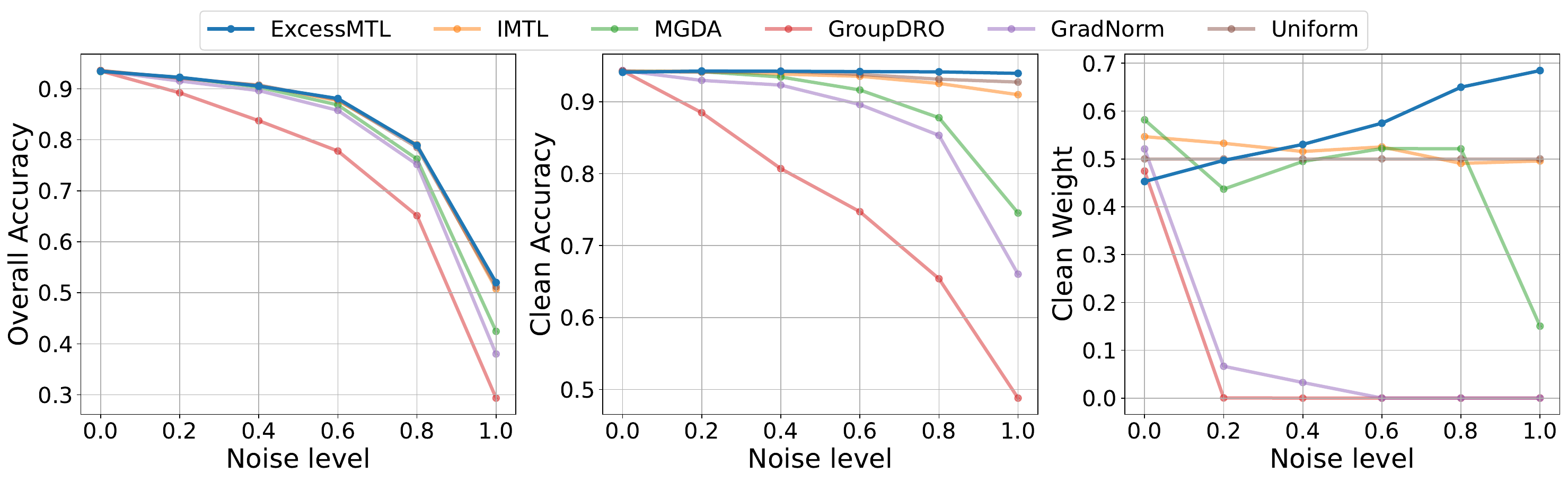}
    \vspace{-0.4cm}
    \caption{Results on the MultiMNIST dataset. Only ExcessMTL assigns smaller weights to the noisy task, maintaining the clean task accuracy, whereas other methods show decreasing performance.}
    % \gz{can we merge these 3 figures?}
    % \han{If we have more space, add one sentence to summarize the conclusion in the caption.}
    \label{fig:mnist}
    % \vspace{-0.55cm}
\end{figure*}

\textbf{Training dynamics and weight assignment.}~We analyze the training dynamics for the algorithms mentioned in \Cref{conceptual_comparison} in \Cref{fig:comparison}. For each algorithm, we examine the change in task weights and test accuracy as training proceeds. GradNorm and GroupDRO rapidly assign substantial weights to the noisy task due to their high losses. IMTL and MGDA, on the other hand, initially allocate significant weights to the noisy task but eventually converge to nearly uniform weighting. Their emphasis on the noisy task dramatically impacts the performance of the clean task, leading to suboptimal performance.

In contrast, ExcessMTL identifies the higher excess risk associated with the clean task at the beginning of training, leading to a higher accumulation of weight on the clean task. Simultaneously, the reduced emphasis on the noisy task helps mitigate overfitting, leading to improved performance when learning with noise. This ability of ExcessMTL to assign appropriate weights based on the proximity to convergence contributes to its superior overall performance. 
% MGDA, on the other hand, employs the Frank-Wolfe algorithm, which produces independent solutions in each round. As a result, the weight updates are abrupt, leading to complete neglect of the noisy task from the outset. Although MGDA achieves high accuracy on the clean tasks, it performs poorly in learning from noise.

\textbf{Performance comparison.}~In \Cref{fig:pareto}, we present the converged performance of all algorithms with and without significant noise injection. When more noise is injected into task 2, its Bayes optimal loss increases, leading to a degradation in its single-task performance. A robust MTL algorithm should retain performance close to Bayes optimal regardless of noise level, i.e., reaching or surpassing the intersection of the single-task performance (dashed line). Despite performing well under the noise-free scenario, all algorithms except ExcessMTL have performance decrease in task 1 under label noise. The performance for loss weighting methods aligns with expectations in \Cref{fig:er_loss}, i.e., they aim for equal losses across both tasks, resulting in undertraining of task 1. \looseness=-1

% since the gradient magnitude of the noisy task is considerably smaller at the beginning of training, it also assigns high weight to the noisy task. As a result, the convergence of MGDA is much slower, as the clean task's gradient magnitude only decreases to be comparable to that of the noisy task after significant training. 
% In contrast, both versions of ExcessMTL identify that the clean task have higher excess risks and focus on, while still being able to learn from the noisy task, resulting in superior overall performance. 

\textbf{Clean validation set.}~One may expect that the existence of a clean validation set would address the label noise issue, as validation performance is a more robust metric than training loss. However, since weight selection on the validation set is performed on the clean data whereas the training set still contains noise, 
% as if no noise exists in the training set either, 
it could lead to assigning considerable weight to noisy tasks. To illustrate this possibility, we compare with MOML \citep{ye2021multiobjective}, which assigns weights using MGDA on a clean validation set. On the Office-Home dataset, we allocate $20\%$ of the training data as a clean validation set and inject noise into the remainder. 

From \Cref{fig:moml}, despite the additional requirement of clean data, MOML remains largely impacted by label noise. When noise level is low, MOML performs worse than MGDA due to less gradient information contained in the validation set than the training set. When noise level is high, although MOML slightly improves over MGDA with more consistent weight assignment, the overall performance still significantly degrades. In contrast, ExcessMTL consistently outperforms both baselines, especially at high noise level.

% To demonstrate the necessity of our algorithm in diverse scenarios,
% it fails to account for noise in the training set
% There exists algorithms that perform weight selection based on the validation set, and one of the most popular one is MOML \citep{ye2021multiobjective}, which formulates MTL as a multi-objective meta learning problem. It assign weights using MGDA on a clean validation set. However, even if a clean validation set is available, MOML is \textit{not} robust to label noise because it fails to account for noise in the training set, still leading to considerable weights for noisy tasks.

 \subsection{Benchmark Evaluation}

% \begin{figure}[t!]
%     \centering
%     \includegraphics[width=\linewidth]{figs/celeba_full.pdf}
%     \caption{Results on the CelebA dataset.}
%     \label{fig:celeba}
% \end{figure}

We present the results on three MTL benchmarks. For MultiMNIST and Office-Home, we present three plots: the overall performance, the average performance on the clean tasks and the sum of task weights assigned to the clean tasks, where the x-axis is the noise level. For NYUv2, we provide the performance of each task since they have different evaluation criteria. Along with adaptive weighting algorithms mentioned in \Cref{conceptual_comparison}, we also include uniform scalarization as a baseline.

The key evaluation criteria for robustness is that \textit{the performance on the clean tasks should not be affected even in face of increasing label noise in the noisy tasks.}

\begin{figure*}[t!]
    \centering
    \includegraphics[width=0.75\linewidth]{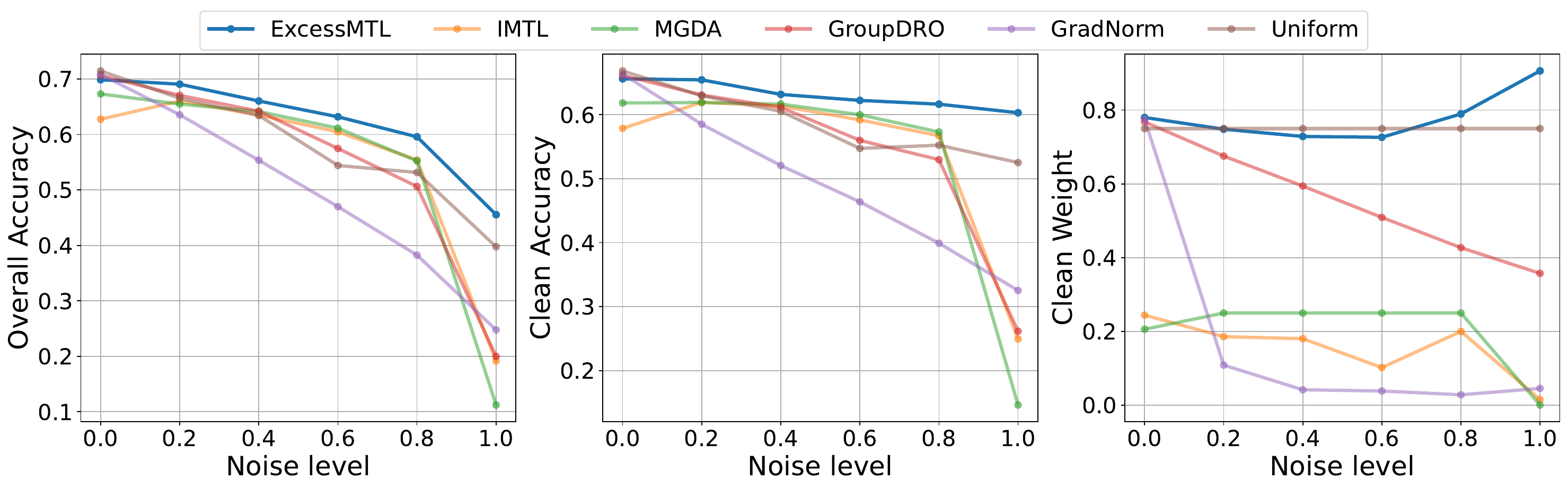}
    \vspace{-0.4cm}
    \caption{Results on the Office-Home dataset (noise in product classification). The left figure considers all tasks, while the other two consider all tasks except product images. The right figure is the combined weights of all clean tasks (0.75 for uniform scalarization). ExcessMTL is least affected by label noise and is the only adaptive algorithm assigning small weights to the noisy task.}
    \label{fig:office}
    % \vspace{-0.5cm}
\end{figure*}

\begin{figure*}[t!]
    \centering
    \includegraphics[width=0.95\linewidth]{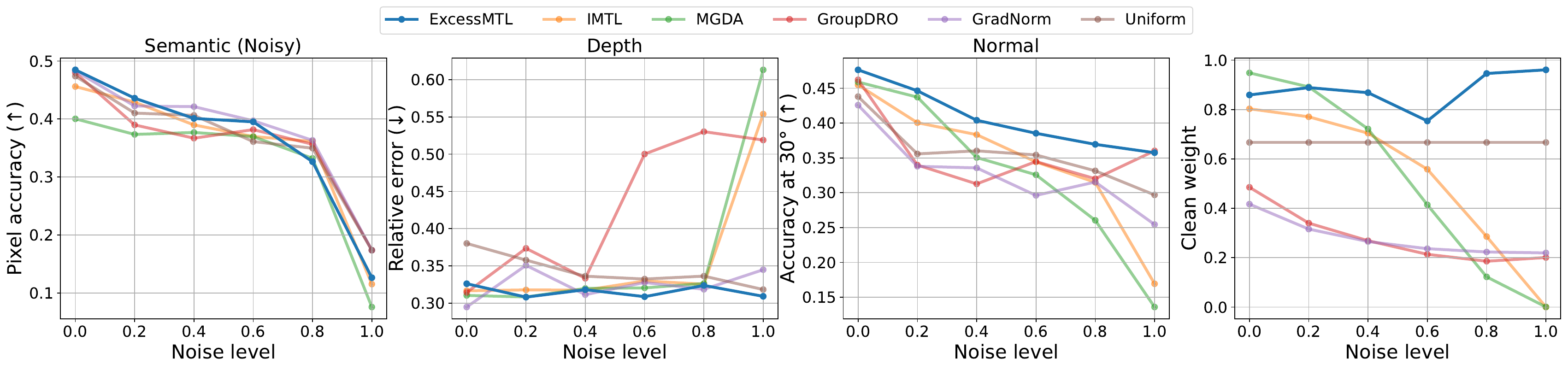}
    \vspace{-0.4cm}
    \caption{Results on the NYUv2 dataset (noise in semantic segmentation). Higher pixel accuracy for semantic segmentation ($\uparrow$), lower relative error for depth estimation ($\downarrow$) and higher angle accuracy for surface normal estimation ($\uparrow$) are desired. Pixel accuracy monotonically decreases as expected. ExcessMTL consistently achieves the best performance on the clean tasks, while other methods fail dramatically under high label noise.}
    \label{fig:nyu}
    \vspace{-0.4cm}
\end{figure*}

\textbf{MultiMNIST.}~From \Cref{fig:mnist}, we can see that for all adaptive weighting methods except ExcessMTL, the performance on the clean task monotonically declines as the noise level increases. The weight plot indicates that this performance degradation stems from the excessive weights assigned to noisy tasks, and in extreme cases, the weight on the clean task is 0, effectively disregarding it during optimization. In contrast, ExcessMTL identifies and assigns minimal weights to the noisy task, preserving the performance on clean tasks even under high label noise (flat line in the clean accuracy plot). 

Surprisingly, uniform scalarization also exhibits resilience to label noise, although slightly inferior to ExcessMTL. This can be attributed to the nature of MultiMNIST as an easy dataset with abundant data in the sense that two tasks have small interference with each other. Therefore, even if one task is fully corrupted, the clean task provides sufficient information to learn a model. However, we find that the performance of uniform scalarization is inconsistent and varies across datasets, as shown by the following experiments.

\textbf{Office-Home.}~Similar trends can be observed in \Cref{fig:office}. The performance of ExcessMTL is least affected on the clean tasks, showcasing its robustness. When the noisy task has purely random label (noise level$=1$), all other adaptive methods completely fail. In contrast to the results in MultiMNIST, uniform scalarization is not sufficient to obtain good performance in this real-world dataset as it is outperformed by most adaptive weighting methods in the noisy setting. This underscores the need for adaptive methods in such scenarios.
% like ExcessMTL 

% This ensures a consistent overall performance regardless of the noise level. 

\textbf{NYUv2.}~From \Cref{fig:nyu}, we observe that ExcessMTL consistently outperforms other methods on the two clean tasks. Other adaptive weighting algorithms assign increasing weights to the noisy task with rising noise level. The phenomenon is particularly evident when the noise level is high, confirming our hypothesis that label noise leads to suboptimal performance for all tasks trained jointly. Since the loss magnitude of the surface normal estimation task is the smallest, GroupDRO and uniform scalarization exhibit clear undertraining on it even without noise injection. Through our scale processing method, ExcessMTL evaluates the excess risks on a comparable scale and trains on all tasks in a more balanced way.

% A similar trend is observed, with GroupDRO and GradNorm exhibiting decreasing performance in depth estimation and surface normal estimation as the noise amount increases in the semantic segmentation task. In addition, both GradNorm and MGDA favor the task with the smallest magnitude of the gradient, i.e., the surface estimation task. Through our proposed scale processing method, ExcessMTL evaluates the excess risks on a comparable scale and trains on all tasks in a more balanced way. 
% Additional results for noise injection in other tasks can be found in Appendix~\ref{supp:exp}. %\han{You can remove this widow word by using $\backslash$sim before the reference.}

In summary, the experiments validate our hypothesis that label noise affects not only the noisy tasks, but also all tasks trained jointly, leading to an overall decline in performance. Compared with other adaptive weighting methods, ExcessMTL consistently achieves high overall performance by appropriately assigning weights to noisy tasks, thereby mitigating the negative impact of label noise. Notably, ExcessMTL maintains competitive performance with other methods in noise-free scenarios, indicating that its robustness does not come at the cost of performance degradation.

% For GradNorm, especially in NYUv2, it favors the task with the smallest magnitude of gradient, i.e., the surface estimation task, thereby ignoring the depth loss, resulting in the worst depth estimation performance among all algorithms. Through the scale processing method proposed in \Cref{algo}, our approach evaluates the excess risks in a comparable scale and trains on all tasks in a more balanced way.

\section{Related Work}

% Multi-task Learning (MTL) poses a challenging optimization problem due to the potential conflicts among tasks~\citep{Crawshaw2020MultiTaskLW}. 
% Here, we introduce different aspects of MTL optimization, the task weighting methods we focus on, and broader applications our algorithm can be useful for.

\textbf{MTL Optimization}~can be tackled from multiple perspectives. One line of work focuses on network architecture design~\citep{dai2016instance,liu2019end}. Another approach is gradient manipulation, where \citet{chen2020just} drops gradient components by the extent of conflict; \citet{yu2020gradient} removes gradient conflict by projection; \citet{liu2021conflict}  minimizes average loss while tracking the improvement on the worst task. Researchers have also studied the problem from a game theory perspective~\citep{nashmtl}. In this study, our primary focus lies on task weighting methods~\citep{kendall2018multi,chen2018gradnorm,sener2018multi,liutowards,Lin2021ReasonableEO}.

\textbf{Task weighting}~is a common strategy employed to prioritize tasks with inferior performance. Existing algorithms in this domain differ in their measures of task difficulty, including homoscedastic uncertainty~\citep{kendall2018multi}, task-specific performance measures~\citep{guo2018dynamic}, norms of gradients and training rates~\citep{chen2018gradnorm}, or a combination of task-specific losses and gradients~\citep{liutowards}. Another approach formulates the MTL problem as a multi-objective optimization problem~\citep{sener2018multi} and utilizes the classical multiple-gradient descent algorithm (MGDA) \citep{mukai,FliegeSvaiter00,desideri2012multiple} to find Pareto stationary solutions. While MGDA does not explicitly use gradient norms for task weighting, it implicitly considers them by seeking the minimum norm point within the convex hull of task gradients.
% We propose to use excess risks as the difficulty measure, as it is robust to label noise.

However, the problem of label noise is often overlooked in the existing works, which can make many proposed task difficulty measures inappropriate. For instance, high label noise will lead to high losses and potentially small gradient magnitudes, making aforementioned algorithms assign high weights to the noisy tasks and causing an imbalance in training. Our work incorporates the idea of prioritizing difficult tasks, while further improves the robustness to label noise. \looseness=-1

\textbf{Loss weighting in other applications.} Beyond MTL, various machine learning problems employ similar mechanism of focusing on challenging samples, often utilizing loss as a difficulty measure. In domain generalization, prior works focus on domains or groups with higher losses to improve generalization~\citep{sagawadistributionally,pmlr-v139-liu21f,piratlafocus}. In hard example mining, researchers use losses to determine whether specific data instances are difficult for the model~\citep{Shrivastava2016TrainingRO,Yuan_2017_ICCV,xue2019hard}. Similar to MTL, those applications also suffer from the label noise problem. We believe our proposed method can offer a more robust difficulty measure for such applications. This extension of our methodology can pave the way for more reliable solutions across a spectrum of machine learning challenges.

% Much like MTL, these applications encounter the label noise problem. We posit that our proposed method can offer a more robust difficulty measure for such applications. By leveraging ExcessMTL's ability to accurately gauge the complexity of tasks, incorporating it into these domains has the potential to significantly enhance the model's performance and generalizability, even in the presence of label noise. This extension of our methodology can pave the way for more effective and reliable solutions across a spectrum of machine learning challenges.

% \textbf{Label noise in MTL.} The problem of learning from label noise has been widely investigated in a single task setting~\citep{song2022survey}, but is often overlooked in the existing MTL works, which can make many of the proposed task difficulty measures inappropriate. For instance, high label noise will lead to high losses and potentially small gradient magnitudes, making aforementioned algorithms assign high weights to the noisy tasks and causing an imbalance in training. Our work incorporates the idea of prioritizing difficult tasks, while further improves the robustness to label noise by introducing a proper measure for task difficulty, i.e., the excess risks.

\section{Conclusion}
In this work, we identify a key limitation of existing adaptive weight updating methods in multi-task learning, i.e., the vulnerability to label noise. Building upon this observation, we propose ExcessMTL, a task balancing algorithm based on excess risks. It employs multiplicative weight update to dynamically adjust the task weights according to their respective distance to convergence. We further show the convergence guarantees of the proposed algorithm and establish connections with multi-objective optimization, showing the Pareto stationarity of its solutions. Extensive experiments across diverse MTL benchmarks demonstrate the consistent superiority of our method in the presence of label noise. Beyond multi-task learning, our insights on excess risks and their connection with convergence distance can potentially inspire more robust algorithmic design in various machine learning applications utilizing loss weighting methods.

\section*{Acknowledgements}
HZ is partially supported by a research grant from the Amazon-Illinois Center on AI for Interactive Conversational Experiences (AICE) and a Google Research Scholar Award.

\section*{Impact Statement}
This paper presents work whose goal is to advance the field of Machine Learning. There are many potential societal consequences of our work, none which we feel must be specifically highlighted here.

% \section*{Acknowledgements}
% This work is partially supported by a research grant from the Amazon-Illinois Center on AI for Interactive Conversational Experiences (AICE) and AWS Cloud Credits.

\bibliography{reference}

\begin{thebibliography}{47}
\providecommand{\natexlab}[1]{#1}
\providecommand{\url}[1]{\texttt{#1}}
\expandafter\ifx\csname urlstyle\endcsname\relax
  \providecommand{\doi}[1]{doi: #1}\else
  \providecommand{\doi}{doi: \begingroup \urlstyle{rm}\Url}\fi

\bibitem[Amari(1998)]{Amari1998NaturalGW}
Amari, S.
\newblock Natural gradient works efficiently in learning.
\newblock \emph{Neural Computation}, 10:\penalty0 251--276, 1998.

\bibitem[Badrinarayanan et~al.(2017)Badrinarayanan, Kendall, and Cipolla]{badrinarayanan2017segnet}
Badrinarayanan, V., Kendall, A., and Cipolla, R.
\newblock Segnet: A deep convolutional encoder-decoder architecture for image segmentation.
\newblock \emph{IEEE transactions on pattern analysis and machine intelligence}, 39\penalty0 (12):\penalty0 2481--2495, 2017.

\bibitem[Bowman~Jr(1976)]{bowman1976relationship}
Bowman~Jr, V.~J.
\newblock On the relationship of the tchebycheff norm and the efficient frontier of multiple-criteria objectives.
\newblock In \emph{Multiple Criteria Decision Making: Proceedings of a Conference Jouy-en-Josas, France May 21--23, 1975}, pp.\  76--86. Springer, 1976.

\bibitem[Burgert et~al.(2022)Burgert, Ravanbakhsh, and Demir]{Burgert_2022}
Burgert, T., Ravanbakhsh, M., and Demir, B.
\newblock On the effects of different types of label noise in multi-label remote sensing image classification.
\newblock \emph{{IEEE} Transactions on Geoscience and Remote Sensing}, 60:\penalty0 1--13, 2022.
\newblock \doi{10.1109/tgrs.2022.3226371}.

\bibitem[Caruana(1997)]{caruana1997multitask}
Caruana, R.
\newblock Multitask learning.
\newblock \emph{Machine learning}, 28:\penalty0 41--75, 1997.

\bibitem[Chen et~al.(2018)Chen, Badrinarayanan, Lee, and Rabinovich]{chen2018gradnorm}
Chen, Z., Badrinarayanan, V., Lee, C.-Y., and Rabinovich, A.
\newblock Gradnorm: Gradient normalization for adaptive loss balancing in deep multitask networks.
\newblock In \emph{International conference on machine learning}, pp.\  794--803. PMLR, 2018.

\bibitem[Chen et~al.(2020)Chen, Ngiam, Huang, Luong, Kretzschmar, Chai, and Anguelov]{chen2020just}
Chen, Z., Ngiam, J., Huang, Y., Luong, T., Kretzschmar, H., Chai, Y., and Anguelov, D.
\newblock Just pick a sign: Optimizing deep multitask models with gradient sign dropout.
\newblock \emph{Advances in Neural Information Processing Systems}, 33:\penalty0 2039--2050, 2020.

\bibitem[Dai et~al.(2016)Dai, He, and Sun]{dai2016instance}
Dai, J., He, K., and Sun, J.
\newblock Instance-aware semantic segmentation via multi-task network cascades.
\newblock In \emph{Proceedings of the IEEE conference on computer vision and pattern recognition}, pp.\  3150--3158, 2016.

\bibitem[D{\'e}sid{\'e}ri(2012)]{desideri2012multiple}
D{\'e}sid{\'e}ri, J.-A.
\newblock Multiple-gradient descent algorithm (mgda) for multiobjective optimization.
\newblock \emph{Comptes Rendus Mathematique}, 350\penalty0 (5-6):\penalty0 313--318, 2012.

\bibitem[Drori \& Shamir(2020)Drori and Shamir]{drori2020complexity}
Drori, Y. and Shamir, O.
\newblock The complexity of finding stationary points with stochastic gradient descent.
\newblock In \emph{International Conference on Machine Learning}, pp.\  2658--2667, 2020.

\bibitem[Duchi et~al.(2011)Duchi, Hazan, and Singer]{duchi2011adaptive}
Duchi, J., Hazan, E., and Singer, Y.
\newblock Adaptive subgradient methods for online learning and stochastic optimization.
\newblock \emph{Journal of machine learning research}, 12\penalty0 (7), 2011.

\bibitem[Fliege \& Svaiter(2000)Fliege and Svaiter]{FliegeSvaiter00}
Fliege, J. and Svaiter, B.~F.
\newblock Steepest descent methods for multicriteria optimization.
\newblock \emph{Mathematical Methods of Operations Research}, 51\penalty0 (3):\penalty0 479--494, 2000.

\bibitem[Guo et~al.(2018)Guo, Haque, Huang, Yeung, and Fei-Fei]{guo2018dynamic}
Guo, M., Haque, A., Huang, D.-A., Yeung, S., and Fei-Fei, L.
\newblock Dynamic task prioritization for multitask learning.
\newblock In \emph{Proceedings of the European conference on computer vision (ECCV)}, pp.\  270--287, 2018.

\bibitem[Hazan et~al.(2016)]{hazan2019introduction}
Hazan, E. et~al.
\newblock Introduction to online convex optimization.
\newblock \emph{Foundations and Trends{\textregistered} in Optimization}, 2\penalty0 (3-4):\penalty0 157--325, 2016.

\bibitem[He et~al.(2016)He, Zhang, Ren, and Sun]{he2016deep}
He, K., Zhang, X., Ren, S., and Sun, J.
\newblock Deep residual learning for image recognition.
\newblock In \emph{Proceedings of the IEEE conference on computer vision and pattern recognition}, pp.\  770--778, 2016.

\bibitem[Hsieh \& Tseng(2021)Hsieh and Tseng]{Hsieh_Tseng_2021}
Hsieh, M.-E. and Tseng, V.
\newblock Boosting multi-task learning through combination of task labels - with applications in ecg phenotyping.
\newblock \emph{Proceedings of the AAAI Conference on Artificial Intelligence}, 35\penalty0 (9):\penalty0 7771--7779, May 2021.

\bibitem[Hu et~al.(2020)Hu, Li, and Yu]{husimple}
Hu, W., Li, Z., and Yu, D.
\newblock Simple and effective regularization methods for training on noisily labeled data with generalization guarantee.
\newblock In \emph{International Conference on Learning Representations}, 2020.

\bibitem[Hu et~al.(2023)Hu, Xian, Wu, Fan, Yin, and Zhao]{hu2023revisiting}
Hu, Y., Xian, R., Wu, Q., Fan, Q., Yin, L., and Zhao, H.
\newblock Revisiting scalarization in multi-task learning: A theoretical perspective.
\newblock \emph{Advances in Neural Information Processing Systems}, 2023.

\bibitem[Kendall et~al.(2018)Kendall, Gal, and Cipolla]{kendall2018multi}
Kendall, A., Gal, Y., and Cipolla, R.
\newblock Multi-task learning using uncertainty to weigh losses for scene geometry and semantics.
\newblock In \emph{Proceedings of the IEEE conference on computer vision and pattern recognition}, pp.\  7482--7491, 2018.

\bibitem[Kim et~al.(2019)Kim, Yim, Yun, and Kim]{kim2019nlnl}
Kim, Y., Yim, J., Yun, J., and Kim, J.
\newblock Nlnl: Negative learning for noisy labels.
\newblock In \emph{Proceedings of the IEEE/CVF international conference on computer vision}, pp.\  101--110, 2019.

\bibitem[Kingma \& Ba(2015)Kingma and Ba]{adam}
Kingma, D.~P. and Ba, J.
\newblock Adam: {A} method for stochastic optimization.
\newblock In Bengio, Y. and LeCun, Y. (eds.), \emph{3rd International Conference on Learning Representations, {ICLR} 2015, San Diego, CA, USA, May 7-9, 2015, Conference Track Proceedings}, 2015.

\bibitem[Kivinen \& Warmuth(1997)Kivinen and Warmuth]{kivinen1997exponentiated}
Kivinen, J. and Warmuth, M.~K.
\newblock Exponentiated gradient versus gradient descent for linear predictors.
\newblock \emph{information and computation}, 132\penalty0 (1):\penalty0 1--63, 1997.

\bibitem[Lin \& Zhang(2023)Lin and Zhang]{lin2023libmtl}
Lin, B. and Zhang, Y.
\newblock {LibMTL}: A {P}ython library for multi-task learning.
\newblock \emph{Journal of Machine Learning Research}, 24\penalty0 (209):\penalty0 1--7, 2023.

\bibitem[Lin et~al.(2022)Lin, Ye, Zhang, and Tsang]{Lin2021ReasonableEO}
Lin, B., Ye, F., Zhang, Y., and Tsang, I. W.-H.
\newblock Reasonable effectiveness of random weighting: A litmus test for multi-task learning.
\newblock \emph{Transactions on Machine Learning Research}, 2022.

\bibitem[Liu et~al.(2021{\natexlab{a}})Liu, Liu, Jin, Stone, and Liu]{liu2021conflict}
Liu, B., Liu, X., Jin, X., Stone, P., and Liu, Q.
\newblock Conflict-averse gradient descent for multi-task learning.
\newblock \emph{Advances in Neural Information Processing Systems}, 34:\penalty0 18878--18890, 2021{\natexlab{a}}.

\bibitem[Liu et~al.(2021{\natexlab{b}})Liu, Haghgoo, Chen, Raghunathan, Koh, Sagawa, Liang, and Finn]{pmlr-v139-liu21f}
Liu, E.~Z., Haghgoo, B., Chen, A.~S., Raghunathan, A., Koh, P.~W., Sagawa, S., Liang, P., and Finn, C.
\newblock Just train twice: Improving group robustness without training group information.
\newblock In Meila, M. and Zhang, T. (eds.), \emph{Proceedings of the 38th International Conference on Machine Learning}, volume 139 of \emph{Proceedings of Machine Learning Research}, pp.\  6781--6792. PMLR, 18--24 Jul 2021{\natexlab{b}}.

\bibitem[Liu et~al.(2021{\natexlab{c}})Liu, Li, Kuang, Xue, Chen, Yang, Liao, and Zhang]{liutowards}
Liu, L., Li, Y., Kuang, Z., Xue, J.-H., Chen, Y., Yang, W., Liao, Q., and Zhang, W.
\newblock Towards impartial multi-task learning.
\newblock In \emph{International Conference on Learning Representations}, 2021{\natexlab{c}}.

\bibitem[Liu et~al.(2019)Liu, Johns, and Davison]{liu2019end}
Liu, S., Johns, E., and Davison, A.~J.
\newblock End-to-end multi-task learning with attention.
\newblock In \emph{Proceedings of the IEEE/CVF conference on computer vision and pattern recognition}, pp.\  1871--1880, 2019.

\bibitem[Michel et~al.(2021)Michel, Ruder, and Yogatama]{michel2021balancing}
Michel, P., Ruder, S., and Yogatama, D.
\newblock Balancing average and worst-case accuracy in multitask learning.
\newblock \emph{arXiv preprint arXiv:2110.05838}, 2021.

\bibitem[Mukai(1980)]{mukai}
Mukai, H.
\newblock Algorithms for multicriterion optimization.
\newblock \emph{IEEE Transactions on Automatic Control}, 25\penalty0 (2):\penalty0 177--186, 1980.
\newblock \doi{10.1109/TAC.1980.1102298}.

\bibitem[Navon et~al.(2022)Navon, Shamsian, Achituve, Maron, Kawaguchi, Chechik, and Fetaya]{nashmtl}
Navon, A., Shamsian, A., Achituve, I., Maron, H., Kawaguchi, K., Chechik, G., and Fetaya, E.
\newblock Multi-task learning as a bargaining game.
\newblock \emph{Proceedings of Machine Learning Research}, 162:\penalty0 16428--16446, 2022.

\bibitem[Nemirovski et~al.(2009)Nemirovski, Juditsky, Lan, and Shapiro]{nemirovski2009robust}
Nemirovski, A., Juditsky, A., Lan, G., and Shapiro, A.
\newblock Robust stochastic approximation approach to stochastic programming.
\newblock \emph{SIAM Journal on optimization}, 19\penalty0 (4):\penalty0 1574--1609, 2009.

\bibitem[Nemirovskij \& Yudin(1983)Nemirovskij and Yudin]{nemirovskij1983problem}
Nemirovskij, A.~S. and Yudin, D.~B.
\newblock Problem complexity and method efficiency in optimization.
\newblock 1983.

\bibitem[Piratla et~al.(2021)Piratla, Netrapalli, and Sarawagi]{piratlafocus}
Piratla, V., Netrapalli, P., and Sarawagi, S.
\newblock Focus on the common good: Group distributional robustness follows.
\newblock In \emph{International Conference on Learning Representations}, 2021.

\bibitem[Sabour et~al.(2017)Sabour, Frosst, and Hinton]{sabour2017dynamic}
Sabour, S., Frosst, N., and Hinton, G.~E.
\newblock Dynamic routing between capsules.
\newblock \emph{Advances in neural information processing systems}, 30, 2017.

\bibitem[Sagawa et~al.(2020)Sagawa, Koh, Hashimoto, and Liang]{sagawadistributionally}
Sagawa, S., Koh, P.~W., Hashimoto, T.~B., and Liang, P.
\newblock Distributionally robust neural networks.
\newblock In \emph{International Conference on Learning Representations}, 2020.

\bibitem[Sener \& Koltun(2018)Sener and Koltun]{sener2018multi}
Sener, O. and Koltun, V.
\newblock Multi-task learning as multi-objective optimization.
\newblock \emph{Advances in neural information processing systems}, 31, 2018.

\bibitem[Shrivastava et~al.(2016)Shrivastava, Gupta, and Girshick]{Shrivastava2016TrainingRO}
Shrivastava, A., Gupta, A.~K., and Girshick, R.~B.
\newblock Training region-based object detectors with online hard example mining.
\newblock \emph{2016 IEEE Conference on Computer Vision and Pattern Recognition (CVPR)}, pp.\  761--769, 2016.

\bibitem[Silberman et~al.(2012)Silberman, Hoiem, Kohli, and Fergus]{silberman2012indoor}
Silberman, N., Hoiem, D., Kohli, P., and Fergus, R.
\newblock Indoor segmentation and support inference from rgbd images.
\newblock \emph{ECCV (5)}, 7576:\penalty0 746--760, 2012.

\bibitem[Steuer \& Choo(1983)Steuer and Choo]{steuer1983interactive}
Steuer, R.~E. and Choo, E.-U.
\newblock An interactive weighted tchebycheff procedure for multiple objective programming.
\newblock \emph{Mathematical programming}, 26:\penalty0 326--344, 1983.

\bibitem[Venkateswara et~al.(2017)Venkateswara, Eusebio, Chakraborty, and Panchanathan]{venkateswara2017deep}
Venkateswara, H., Eusebio, J., Chakraborty, S., and Panchanathan, S.
\newblock Deep hashing network for unsupervised domain adaptation.
\newblock In \emph{Proceedings of the IEEE Conference on Computer Vision and Pattern Recognition}, pp.\  5018--5027, 2017.

\bibitem[Vijayakumar \& Schaal(2000)Vijayakumar and Schaal]{vijayakumar2000locally}
Vijayakumar, S. and Schaal, S.
\newblock Locally weighted projection regression: An o (n) algorithm for incremental real time learning in high dimensional space.
\newblock In \emph{Proceedings of the seventeenth international conference on machine learning (ICML 2000)}, volume~1, pp.\  288--293. Morgan Kaufmann, 2000.

\bibitem[Xue et~al.(2019)Xue, Han, Zheng, Guo, and Wu]{xue2019hard}
Xue, J., Han, J., Zheng, T., Guo, J., and Wu, B.
\newblock Hard sample mining for the improved retraining of automatic speech recognition, 2019.

\bibitem[Ye et~al.(2021)Ye, Lin, Yue, Guo, Xiao, and Zhang]{ye2021multiobjective}
Ye, F., Lin, B., Yue, Z., Guo, P., Xiao, Q., and Zhang, Y.
\newblock Multi-objective meta learning.
\newblock In Beygelzimer, A., Dauphin, Y., Liang, P., and Vaughan, J.~W. (eds.), \emph{Advances in Neural Information Processing Systems}, 2021.
\newblock URL \url{https://openreview.net/forum?id=wKf9iSu_TEm}.

\bibitem[Yu et~al.(2020)Yu, Kumar, Gupta, Levine, Hausman, and Finn]{yu2020gradient}
Yu, T., Kumar, S., Gupta, A., Levine, S., Hausman, K., and Finn, C.
\newblock Gradient surgery for multi-task learning.
\newblock \emph{Advances in Neural Information Processing Systems}, 33:\penalty0 5824--5836, 2020.

\bibitem[Yuan et~al.(2017)Yuan, Yang, and Zhang]{Yuan_2017_ICCV}
Yuan, Y., Yang, K., and Zhang, C.
\newblock Hard-aware deeply cascaded embedding.
\newblock In \emph{Proceedings of the IEEE International Conference on Computer Vision (ICCV)}, Oct 2017.

\bibitem[Zhou et~al.(2022)Zhou, Zhang, Jiang, Zhong, Gu, and Zhu]{zhou2022convergence}
Zhou, S., Zhang, W., Jiang, J., Zhong, W., Gu, J., and Zhu, W.
\newblock On the convergence of stochastic multi-objective gradient manipulation and beyond.
\newblock \emph{Advances in Neural Information Processing Systems}, 35:\penalty0 38103--38115, 2022.

\end{thebibliography}
\bibliographystyle{icml2024}

%%%%%%%%%%%%%%%%%%%%%%%%%%%%%%%%%%%%%%%%%%%%%%%%%%%%%%%%%%%%%%%%%%%%%%%%%%%%%%%
%%%%%%%%%%%%%%%%%%%%%%%%%%%%%%%%%%%%%%%%%%%%%%%%%%%%%%%%%%%%%%%%%%%%%%%%%%%%%%%
% APPENDIX
%%%%%%%%%%%%%%%%%%%%%%%%%%%%%%%%%%%%%%%%%%%%%%%%%%%%%%%%%%%%%%%%%%%%%%%%%%%%%%%
%%%%%%%%%%%%%%%%%%%%%%%%%%%%%%%%%%%%%%%%%%%%%%%%%%%%%%%%%%%%%%%%%%%%%%%%%%%%%%%
\newpage
\appendix
\onecolumn

% We organize the supplemental materials as follows. \Cref{proof:pareto_opt} provides the proof for Pareto optimality. \Cref{proof:pareto_stationary} provides the proof for Pareto stationary. \Cref{supp:exp} supplements more experimental details and results.
% \end{abstract}

\section{Proof for Pareto Optimality}\label{proof:pareto_opt}
\convergence*
\begin{proof}
    We directly apply the well-established regret bound for online mirror descent from~\citet{nemirovski2009robust} on the saddle point problem
    \begin{align*}
        \min_{\theta\in\Theta} \max_{\alpha\in\Delta_m} \sum_{i=1}^m \alpha_i f_i(\theta).
    \end{align*}
    In our case, $f_i(\cdot)$ is the excess risk for the $i$-th task $\cE_i(\cdot)$. We take inspiration from the proof strategy of Proposition 2 in~\cite{sagawadistributionally} to first present the bound and then explain why our formulation satisfies all conditions for the bound.
    \begin{assumption}\label{as:convex}
        $f_i$ is convex on $\Theta$.
    \end{assumption}
    \begin{assumption}\label{as:func}
        Let $\xi$ be a random vector that takes value in $\Xi$. For all $i\in[m]$, there exists a function $F_i:\Theta\times\Xi\rightarrow\bR$ such that $\E_{\xi\sim p}[F_i(\theta;\xi)]=f_i(\theta)$.
    \end{assumption}
    \begin{assumption}\label{as:grad}
        For every given $\theta\in\Theta$ and $\xi\in\Xi$, we are able to compute $F_i(\theta,\xi)$ and the subgradient $\nabla F_i(\theta,\xi)$ such that $\E_{\xi\sim p}[\nabla F_i(\theta,\xi)]=\nabla f_i(\theta)$, $\E_{\xi\sim p}[\|\nabla F_i(\theta,\xi)-\nabla f_i(\theta)\|]\leq \sigma$.
    \end{assumption}    
    \begin{theorem}[\citealt{nemirovski2009robust}]
        If Assumption 1-3 hold, at training step $t$, the regret for the online mirror descent algorithm on the saddle point problem
        \begin{align*}
            \min_{\theta\in\Theta} \max_{\alpha\in\Delta_m} \sum_{i=1}^m \alpha_i f_i(\theta)
        \end{align*}
        can be bounded by
        \begin{align*}
            \E\left[ \max_{\alpha\in\Delta_m} \sum_{i=1}^m \alpha_i f_i(\bar{\theta}^{(1:t)}) - \min_{\theta\in\Theta} \sum_{i=1}^m \bar{\alpha_i}^{(1:t)} f_i(\theta) \right] \leq 2\sqrt{\frac{10(R_\theta^2 M^2_{*,\theta} + M^2_{*,\alpha} \log m)}{t}},
        \end{align*}
        where
        \begin{align*}
            &\E \left[ \left\| \nabla_\theta \sum_{i=1}^m \alpha_i F_i(\theta;\xi) \right\|^2_{*,\theta} \right] \leq M_{*,\theta},\\
            &\E \left[ \left\| \nabla_\alpha \sum_{i=1}^m \alpha_i F_i(\theta;\xi) \right\|^2_{*,\alpha} \right] \leq M_{*,\alpha },\\
            &R_\theta^2=\frac{1}{c}(\max_\theta \| \theta \|^2_\theta - \min_\theta \| \theta \|^2_\theta)
        \end{align*}
        for a $c-$strongly convex norm $\| \cdot \|_\theta$.
    \end{theorem}
    Assumption \ref{as:convex} is the same as our condition (ii). For Assumption \ref{as:func}, we can let $\xi$ be the tuple $(x,y,i)$, where $i$ is the task index. Then, let the distribution of $\xi$ can be a mixture of each task distribution $P_i$, i.e.,
    \begin{align*}
        p\coloneqq \frac{1}{m}\sum_{i=1}^m P_i.
    \end{align*}
    To make $F_i$ an unbiased estimator for $f_i$, we construct
    \begin{align*}
        F_i(\theta;(x,y,i')) \coloneqq m\indicator[i=i'] f_i(\theta).
    \end{align*}
    We can check the validity of this construction by
    \begin{align*}
        \E_{(x,y,i')\sim p}[F_i(\theta;(x,y,i'))] = \frac{1}{m} \E_{P_i}[mf_i(\theta)] = f_i(\theta).
    \end{align*}
    For Assumption \ref{as:grad}, similarly,
    \begin{align*}
        \E_{(x,y,i')\sim p}[\nabla F_i(\theta;(x,y,i'))] = \frac{1}{m} \E_{P_i}[m\nabla f_i(\theta)] = \nabla f_i(\theta).
    \end{align*}
    According to our condition (iii) and (iv), we have that
    \begin{align*}
        &\E \left[ \left\| \nabla_\theta \sum_{i=1}^m \alpha_i F_i(\theta;(x,y,i')) \right\|^2_{*,\theta} \right] \leq m^2L^2 = M_{*,\theta},\\
        &\E \left[ \left\| \nabla_\alpha \sum_{i=1}^m \alpha_i F_i(\theta;\xi) \right\|^2_{*,\alpha} \right] \leq m^2B_\ell^2 = M_{*,\alpha },\\
        &R_\theta^2=\frac{1}{c}(\max_\theta \| \theta \|^2_\theta - \min_\theta \| \theta \|^2_\theta) = B_\theta^2.
    \end{align*}
    Therefore, we obtain
    \begin{align}
        \sum_{i=1}^m \alpha_i^{(t)}\cE_i(\bar{\theta}^{(1:t)}) - \min_{\theta\in\Theta} \max_{\alpha\in\Delta_m} \sum_{i=1}^m \alpha_i\cE_i(\theta) &\leq 
        \max_{\alpha\in\Delta_m} \sum_{i=1}^m \alpha_i f_i(\bar{\theta}^{(1:t)}) - \min_{\theta\in\Theta} \sum_{i=1}^m \bar{\alpha_i}^{(1:t)} f_i(\theta), \nonumber\\ 
        &\leq 2m\sqrt{\frac{10(B_\theta^2 L^2 + B_\ell^2 \log m)}{t}}.\label{eq:bound}
    \end{align}
\end{proof}

\pareto*
\begin{proof}
    % One example of the weights satisfying the above condition is $\bar{\alpha}^{(1:t)}$ because
    % \begin{align*}
    %     \sum_{i=1}^m \bar{\alpha_i}^{(1:t)}\cE_i(\bar{\theta}^{(1:t)}) - \min_{\theta\in\Theta} \sum_{i=1}^m \bar{\alpha_i}^{(1:t)}\cE_i(\theta)
    %     & \leq \max_{\alpha\in\Delta_m} \sum_{i=1}^m \alpha_i f_i(\bar{\theta}^{(1:t)}) - \min_{\theta\in\Theta} \sum_{i=1}^m \bar{\alpha_i}^{(1:t)} f_i(\theta)\\
    %     &\leq 2m\sqrt{\frac{10(B_\theta^2 L^2 + B_\ell^2 \log m)}{t}}.
    % \end{align*}

    From \Cref{eq:bound}, it is clear that $\alpha^{(t)}$ satisfies the above condition because
    \begin{align*}
        \sum_{i=1}^m \alpha_i^{(t)}\cE_i(\bar{\theta}^{(1:t)}) - \min_{\theta\in\Theta} \sum_{i=1}^m \alpha_i^{(t)}\cE_i(\theta)
        & \leq \max_{\alpha\in\Delta_m} \sum_{i=1}^m \alpha_i f_i(\bar{\theta}^{(1:t)}) - \min_{\theta\in\Theta} \sum_{i=1}^m \alpha_i^{(t)} f_i(\theta)\\
        &\leq 2m\sqrt{\frac{10(B_\theta^2 L^2 + B_\ell^2 \log m)}{t}}.
    \end{align*}
\end{proof}

\section{Proof for Pareto Stationary}\label{proof:pareto_stationary}
\paretostationary*
Before proving Theorem~\ref{thm:convergence_nonconvex}, we first present following necessary lemmas. Note that we rewrite $F_i ( \theta^{(t)}; \xi)$ as $F_i ( \theta^{(t)})$ for simplicity in the next context. Without loss of generality, we assume that $\cE_i({\theta}^{(t)})$ is bounded by $B_\ell$.

\begin{lemma}\label{lem:dot} Under the same assumption in Theorem~\ref{thm:convergence_nonconvex}, select nonincreasing ${\eta_{\theta}^{(t)}} \leq \min\{1/B_\ell,1/Gm\}$, we have the following inequality
    \begin{align*}
        & \mathbb{E}_{ \xi}\left[\left(\sum_{i=1}^m  \alpha_{i}^{(t)}\nabla f_i ( \theta^{(t)}) \right)^\top  \left(-\sum_{i=1}^m  \alpha_{i}^{(t)}\nabla F_i ( \theta^{(t)}) \right)\right]  \leq 4 m^{3/2} L \sigma {\eta_{\alpha}^{(t)}} B_{\ell} - \mathbb{E}_{ \xi}\left[\left\|\sum_{i=1}^m  \alpha_{i}^{(t)}\nabla f_i ( \theta^{(t)})\right\|_2^2 \right].
    \end{align*}
\end{lemma}
\begin{proof} We first decompose the term into
    \begin{align*}
        & \mathbb{E}_{ \xi}\left[\left(\sum_{i=1}^m  \alpha_{i}^{(t)}\nabla f_i ( \theta^{(t)}) \right)^\top  \left(-\sum_{i=1}^m  \alpha_{i}^{(t)}\nabla F_i ( \theta^{(t)}) \right)\right] \\ & = \mathbb{E}_{ \xi}\left[\left(\sum_{i=1}^m  \alpha_{i}^{(t)}\nabla f_i ( \theta^{(t)}) \right)^\top \left( \sum_{i=1}^m  \alpha_{i}^{(t)}\left(\nabla f_i ( \theta^{(t)}) -\nabla F_i ( \theta^{(t)})\right)\right)\right] - \mathbb{E}_{ \xi}\left[\left\|\sum_{i=1}^m  \alpha_{i}^{(t)}\nabla f_i ( \theta^{(t)})\right\|_2^2 \right].
    \end{align*}
    The first term can potentially corrupt the effectiveness of optimization, and the second term measures the descent value. We next bound the first term. By definitions and decomposition, we get
    \begin{align*}
        & \mathbb{E}_{ \xi} \left[ \left(\sum_{i=1}^m \alpha_{i}^{(t)} \nabla f_i ( \theta^{(t)})\right)^\top \left(\sum_{i=1}^m \alpha_{i}^{(t)} (\nabla f_i ( \theta^{(t)}) -  \nabla F_i( \theta^{(t)}) )\right)\right]\\
        & = \mathbb{E}_{ \xi} \left[ \left(\sum_{i=1}^m \alpha_{i}^{(t)} \nabla f_i ( \theta^{(t)})\right)^\top \left(\sum_{i=1}^m (\alpha_{i}^{(t)} - \mathbb{E}[\alpha_{i}^{(t)}]) (\nabla f_i ( \theta^{(t)}) -  \nabla F_i( \theta^{(t)}))\right)\right] \text{    (term A)}\\
        & \quad \quad  + \mathbb{E}_{ \xi} \left[\left(\sum_{i=1}^m  (\alpha_{i}^{(t)} - \mathbb{E}[\alpha_{i}^{(t)}]) \nabla f_i ( \theta^{(t)})\right)^\top \left(\sum_{i=1}^m \mathbb{E}[\alpha_{i}^{(t)}] (\nabla f_i ( \theta^{(t)}) -  \nabla F_i( \theta^{(t)}))\right)\right] \text{    (term B)}\\
        & \quad \quad  + \mathbb{E}_{ \xi} \left[\left(\sum_{i=1}^m  \mathbb{E}[\alpha_{i}^{(t)}] \nabla f_i ( \theta^{(t)})\right)^\top \left(\sum_{i=1}^m \mathbb{E}[\alpha_{i}^{(t)}] (\nabla f_i ( \theta^{(t)}) -  \nabla F_i( \theta^{(t)}))\right) \right] \text{    (term C)}.
    \end{align*}
    We then bound each term individually. As we know that $\alpha_{i}^{(t)}\in[0,1],\|\nabla f_i( \theta^{(t)}) \|\leq L,\forall i= 1,\ldots,m$. By Cauchy–Schwartz inequality, we further know that for the term A
    \begin{align*}
         \text{term A} & =  \mathbb{E}_{ \xi} \left[ \left(\sum_{i=1}^m \alpha_{i}^{(t)} \nabla f_i ( \theta^{(t)})\right)^\top \left(\sum_{i=1}^m (\alpha_{i}^{(t)} - \mathbb{E}[\alpha_{i}^{(t)}]) (\nabla f_i ( \theta^{(t)}) -  \nabla F_i( \theta^{(t)}))\right)\right]\\
         & \leq \mathbb{E}_{ \xi} \left[\left\|\sum_{i=1}^m \alpha_{i}^{(t)} \nabla f_i ( \theta^{(t)})\right\|_2 \left\|\sum_{i=1}^m (\alpha_{i}^{(t)} - \mathbb{E}[\alpha_{i}^{(t)}]) ( \nabla f_i ( \theta^{(t)}) -  \nabla F_i( \theta^{(t)}))\right\|_2\right]\\
         & \leq \mathbb{E}_{ \xi} \left[\left(\sum_{i=1}^m \alpha_{i}^{(t)} \left\| \nabla f_i ( \theta^{(t)})\right\|_2\right) \left\|\sum_{i=1}^m (\alpha_{i}^{(t)} - \mathbb{E}[\alpha_{i}^{(t)}]) ( \nabla f_i ( \theta^{(t)}) -  \nabla F_i( \theta^{(t)}))\right\|_2\right]\\
         & \leq \mathbb{E}_{ \xi} \left[L \left\|\sum_{i=1}^m (\alpha_{i}^{(t)} - \mathbb{E}[\alpha_{i}^{(t)}]) ( \nabla f_i ( \theta^{(t)}) -  \nabla F_i( \theta^{(t)}))\right\|_2\right]\\
         & \leq L \mathbb{E}_{ \xi} \left[\sum_{i=1}^m |\alpha_{i}^{(t)} - \mathbb{E}[\alpha_{i}^{(t)}]|\|  \nabla f_i ( \theta^{(t)}) -  \nabla F_i( \theta^{(t)})\|_2\right].
    \end{align*}
    By the fact that $ab\leq \frac{1}{2{\beta^{(t)}}} a^2 + \frac{{\beta^{(t)}}}{2} b^2$ for any ${\beta^{(t)}} >0$, and by the linearity of expectation, we can get
    \begin{align*}
        \text{term A} & \leq \frac{L}{2{\beta^{(t)}}} \sum_{i=1}^m \mathbb{E}_{ \xi} \left[|\alpha_{i}^{(t)} - \mathbb{E}[\alpha_{i}^{(t)}]|^2\right] + \frac{ L{\beta^{(t)}}}{2}\sum_{i=1}^m  \mathbb{E}_{ \xi} \left[ \|  \nabla f_i ( \theta^{(t)}) -  \nabla F_i( \theta^{(t)})\|_2^2\right]\\
        & \leq \frac{L}{2{\beta^{(t)}}} \sum_{i=1}^m \mathbb{E}_{ \xi} \left[|\alpha_{i}^{(t)} - \alpha_{i}^{(t+1)}|^2\right]  + \frac{L{\beta^{(t)}}}{2}  m\sigma^2\\
        & \leq \frac{L}{2{\beta^{(t)}}} \sum_{i=1}^m (\exp{(\eta_{\alpha}^{(t)}\cE_i({\theta}^{(t)}))}-1)^2 + \frac{L{\beta^{(t)}}}{2}  m\sigma^2\\
        & \leq \frac{L}{2{\beta^{(t)}}} \sum_{i=1}^m (2\eta_{\alpha}^{(t)}\cE_i({\theta}^{(t)}))^2 + \frac{L{\beta^{(t)}}}{2}  m\sigma^2\\
        & \leq \frac{2L}{{\beta^{(t)}}} m {\eta_{\alpha}^{(t)}}^2 B_{\ell}^2 + \frac{L{\beta^{(t)}}}{2}  m\sigma^2.
    \end{align*}
    The third inequality is by the fact that $|\alpha_{i}^{(t)} - \alpha_{i}^{(t+1)}|\leq |\alpha_{i}^{(t)} - \exp{(\eta_{\alpha}^{(t)}\cE_i({\theta}^{(t)}))}\alpha_{i}^{(t)}|\leq (\exp{(\eta_{\alpha}^{(t)}\cE_i({\theta}^{(t)}))}-1)\alpha_{i}^{(t)}$ and $\alpha_{i}^{(t)} \in [0,1]$. The fourth inequality is because $\exp(x) \leq 2x + 1$ when $x\in [0,1]$, and we know that $\eta_{\alpha}^{(t)}\cE_i({\theta}^{(t)} \in [0,1]$.
    By setting ${\beta^{(t)}} = 2{\eta_{\alpha}^{(t)}} B_{\ell}/\sigma$, denote $\mathbb{V}[ \mathbf{\alpha}^{(t)}] = \sum_{i=1}^m \mathbb{E}_{ \xi} \left[|\alpha_{i}^{(t)} - \mathbb{E}[\alpha_{i}^{(t)}]|^2\right]$, we have
    \begin{align*}
        \text{term A} \leq 2 m L \sigma {\eta_{\alpha}^{(t)}} B_{\ell}.
    \end{align*}
    With similar tricks, we have for the term B
    \begin{align*}
        \text{term B} &= \mathbb{E}_{ \xi} \left[\left(\sum_{i=1}^m  (\alpha_{i}^{(t)} - \mathbb{E}[\alpha_{i}^{(t)}]) \nabla f_i ( \theta^{(t)})\right)^\top \left(\sum_{i=1}^m \mathbb{E}[\alpha_{i}^{(t)}] ( \nabla f_i ( \theta^{(t)}) -  \nabla F_i( \theta^{(t)}))\right)\right]\\
        & \leq \mathbb{E}_{ \xi} \left[\left\|\sum_{i=1}^m  (\alpha_{i}^{(t)} - \mathbb{E}[\alpha_{i}^{(t)}]) \nabla f_i ( \theta^{(t)})\right\|_2 \left\|\sum_{i=1}^m \mathbb{E}[\alpha_{i}^{(t)}] ( \nabla f_i ( \theta^{(t)}) -  \nabla F_i( \theta^{(t)}))\right\|_2\right]\\
        &  {\leq \mathbb{E}_{ \xi} \left[\left(\sum_{i=1}^m L |\alpha_{i}^{(t)} - \mathbb{E}[\alpha_{i}^{(t)}]|\right) \left\|\sum_{i=1}^m \mathbb{E}[\alpha_{i}^{(t)}] ( \nabla f_i ( \theta^{(t)}) -  \nabla F_i( \theta^{(t)}))\right\|_2\right]}\\
        & = \mathbb{E}_{ \xi} \left[\sum_{i=1}^m L |\alpha_{i}^{(t)} - \mathbb{E}[\alpha_{i}^{(t)}]| \left\|\sum_{j=1}^m \mathbb{E}[\alpha_{j}^{(t)} ] ( \nabla f^j ( \theta^{(t)}) -  F_j( \theta^{(t)}))\right\|_2\right]\\
        & \leq L \mathbb{E}_{ \xi} \left[\left(\sum_{i=1}^m \frac{1}{2{\beta^{(t)}}}|\alpha_{i}^{(t)} - \mathbb{E}[\alpha_{i}^{(t)}]|^2 + \frac{{\beta^{(t)}} }{2}\left\|\sum_{i=1}^m \mathbb{E}[\alpha_{i}^{(t)}] ( \nabla f_i ( \theta^{(t)}) -  \nabla F_i( \theta^{(t)}))\right\|_2^2\right)\right]\\
        & \leq L \mathbb{E}_{ \xi} \left[\left(\sum_{i=1}^m \frac{1}{2{\beta^{(t)}}}|\alpha_{i}^{(t)} - \mathbb{E}[\alpha_{i}^{(t)}]|^2 + \frac{{\beta^{(t)}} }{2}(\sum_{i=1}^m \mathbb{E}[\alpha_{i}^{(t)}]^2)(\sum_{i=1}^m\left\|\nabla f_i ( \theta^{(t)}) -  \nabla F_i( \theta^{(t)})\right\|_2^2)\right)\right]\\
        & \leq \frac{L}{2{\beta^{(t)}}} \sum_{i=1}^m  \mathbb{E}_{ \xi} \left[|\alpha_{i}^{(t)} - \mathbb{E}[\alpha_{i}^{(t)}]|^2\right] +  \frac{mL{\beta^{(t)}} }{2}\sum_{i=1}^m\mathbb{E}_{ \xi} \left[ \|  \nabla f_i ( \theta^{(t)}) -  \nabla F_i( \theta^{(t)})\|_2^2\right]\\
        & {= \frac{2L}{{\beta^{(t)}}}  m {\eta_{\alpha}^{(t)}}^2 B_{\ell}^2+  \frac{m^2 \sigma^2 L{\beta^{(t)}} }{2}.}
    \end{align*}
    The first inequality is by Cauchy–Schwarz inequality. The second one is by the triangle inequality of $l_2$ norm, we have $\|\sum_{i=1}^m  (\alpha_{i}^{(t)} - \mathbb{E}[\alpha_{i}^{(t)}]) \nabla f_i ( \theta^{(t)})\|_2 \leq \sum_{i=1}^m\|  (\alpha_{i}^{(t)} - \mathbb{E}[\alpha_{i}^{(t)}]) \nabla f_i ( \theta^{(t)})\|_2 = \sum_{i=1}^m|\alpha_{i}^{(t)} - \mathbb{E}[\alpha_{i}^{(t)}]| \| \nabla f_i ( \theta^{(t)})\|_2$. As we know that $\sum_{i=1}^m|\alpha_{i}^{(t)} - \mathbb{E}[\alpha_{i}^{(t)}]| \| \nabla f_i ( \theta^{(t)})\|_2\leq L\sum_{i=1}^m|\alpha_{i}^{(t)} - \mathbb{E}[\alpha_{i}^{(t)}]|$, the third one is from the fact that $ab \leq \frac{1}{2{\beta^{(t)}}} a^2 + \frac{{\beta^{(t)}}}{2} b^2$. The forth one is also from Cauchy-Schwartz inequality. The last one is by the fact that $ \mathbb{E} [\alpha_{i}^{(t)}] \leq 1$. 
    Further by setting ${\beta^{(t)}} = 2{{\eta_{\alpha}^{(t)}} B_{\ell}}/{{\sigma}\sqrt{m}}$, we have
    \begin{align*}
        \text{term B} \leq 2 m^{3/2} L \sigma {\eta_{\alpha}^{(t)}} B_{\ell}.
    \end{align*}
    For term C, by the fact that only $ \nabla F_i( \theta^{(t)})$ has randomness, we get
    \begin{align*}
        \text{term C} &= \mathbb{E}_{ \xi} \left[\left(\sum_{i=1}^m  \mathbb{E}[\alpha_{i}^{(t)}] \nabla f_i ( \theta^{(t)})\right)^\top \left(\sum_{i=1}^m \mathbb{E}[\alpha_{i}^{(t)}] ( \nabla f_i ( \theta^{(t)}) -  \nabla F_i( \theta^{(t)}))\right) \right] \\
        &= \left(\sum_{i=1}^m  \mathbb{E}[\alpha_{i}^{(t)}] \nabla f_i ( \theta^{(t)})\right)^\top \left(\sum_{i=1}^m \mathbb{E}[\alpha_{i}^{(t)}] \left(\nabla f_i ( \theta^{(t)}) -\mathbb{E}_{ \xi}[ \nabla F_i( \theta^{(t)})] \right)\right) =0.
    \end{align*}
    By summing up the above results, we obtain
    \begin{align*}
        \mathbb{E}_{ \xi} \left[ \left(\sum_{i=1}^m \alpha_{i}^{(t)} \nabla f_i ( \theta^{(t)})\right)^\top \left( \sum_{i=1}^m  \alpha_{i}^{(t)}\nabla F_i ( \theta^{(t)})  - \sum_{i=1}^m \alpha_{i}^{(t)} \nabla f_i ( \theta^{(t)})\right)\right]  \leq 4 m^{3/2} L \sigma {\eta_{\alpha}^{(t)}} B_{\ell}.
    \end{align*}
    Plug in the first decomposition in the beginning, we finally get
    \begin{align*}
        & \mathbb{E}_{ \xi}\left[\left(\sum_{i=1}^m  \alpha_{i}^{(t)}\nabla f_i ( \theta^{(t)}) \right)^\top  \left(-\sum_{i=1}^m  \alpha_{i}^{(t)}\nabla F_i ( \theta^{(t)}) \right)\right]  \leq 4 m^{3/2} L \sigma {\eta_{\alpha}^{(t)}} B_{\ell} - \mathbb{E}_{ \xi}\left[\left\|\sum_{i=1}^m  \alpha_{i}^{(t)}\nabla f_i ( \theta^{(t)})\right\|_2^2 \right].
    \end{align*}
\end{proof}

\begin{lemma}\label{lem:second_bias} Under the same assumption in Theorem~\ref{thm:convergence_nonconvex}, select nonincreasing ${\eta_{\theta}^{(t)}} \leq \min\{1/B_\ell,1/Gm\}$, we have the following inequality
    \begin{align*}
        \mathbb{E}_{ \xi} \left[\| \sum_{i=1}^m  \alpha_{i}^{(t)}\nabla F_i ( \theta^{(t)}) \|_2^2 - \|\sum_{i=1}^m \alpha_{i}^{(t)} \nabla f_i ( \theta^{(t)})\|_2^2\right] \leq m\sigma^2 + 4 m^{3/2} L \sigma {\eta_{\alpha}^{(t)}} B_{\ell}.
    \end{align*}
\end{lemma}
\begin{proof} We first decompose the expectation into
    \begin{align*}
        & \mathbb{E}_{ \xi} \left[\| \sum_{i=1}^m  \alpha_{i}^{(t)}\nabla F_i ( \theta^{(t)}) \|_2^2 - \|\sum_{i=1}^m \alpha_{i}^{(t)} \nabla f_i ( \theta^{(t)})\|_2^2\right]  \\
        & = \mathbb{E}_{ \xi} \left[\| \sum_{i=1}^m  \alpha_{i}^{(t)}\nabla F_i ( \theta^{(t)})  +\sum_{i=1}^m \alpha_{i}^{(t)} \nabla f_i ( \theta^{(t)})-\sum_{i=1}^m \alpha_{i}^{(t)} \nabla f_i ( \theta^{(t)}) \|_2^2 - \| {\sum_{i=1}^m}\alpha_{i}^{(t)} \nabla f_i ( \theta^{(t)})\|_2^2\right]\\
        & = \underbrace{\mathbb{E}_{ \xi} \left[\| \sum_{i=1}^m  \alpha_{i}^{(t)}\nabla F_i ( \theta^{(t)})  +\sum_{i=1}^m \alpha_{i}^{(t)} \nabla f_i ( \theta^{(t)})\|_2^2\right]}_{\text{term A}} \\
        & \quad - \underbrace{ 2\mathbb{E}_{ \xi} \left[\left(\sum_{i=1}^m \alpha_{i}^{(t)} \nabla f_i ( \theta^{(t)})\right)^\top \left( \left(\sum_{i=1}^m  \alpha_{i}^{(t)}\nabla F_i ( \theta^{(t)}) \right) +\sum_{i=1}^m \alpha_{i}^{(t)} \nabla f_i ( \theta^{(t)})\right)\right]}_{\text{term B}}.
    \end{align*}
    We next analyze the term A. By the triangle inequality of the $l_2$ norm, we have
    \begin{align*}
        \text{term A} & = \mathbb{E}_{ \xi}\left[\left\|\sum_{i=1}^m  \alpha_{i}^{(t)}\nabla f_i ( \theta^{(t)})  +  \left(\sum_{i=1}^m  \alpha_{i}^{(t)}\nabla F_i ( \theta^{(t)}) \right)\right\|_2^2\right] = \mathbb{E}_{ \xi}\left[\left\|\sum_{i=1}^m  \alpha_{i}^{(t)}(\nabla f_i ( \theta^{(t)})  -  \nabla F_i( \theta^{(t)}))\right\|_2^2\right]\\
        & \leq \mathbb{E}_{ \xi}\left[\left(\sum_{i=1}^m  \alpha_{i}^{(t)} \left\|(\nabla f_i ( \theta^{(t)})  -  \nabla F_i( \theta^{(t)}))\right\|_2\right)^2\right],
    \end{align*}
    Further, by the fact that  {$\sum_{i=1}^m \alpha_{i}^{(t)} =1, \alpha_{i}^{(t)} \in [0,1],i=1,\ldots,m$}, we know that
    \begin{align*}
        \text{term A} &\leq  \mathbb{E}_{ \xi}\left[\left(\sum_{i=1}^m{{\alpha_{i}^{(t)}}^2}\right) \left(\sum_{i=1}^m\left\|\nabla f_i ( \theta^{(t)})  -  \nabla F_i( \theta^{(t)})\right\|^2_2\right)\right]\\
        & \leq \mathbb{E}_{ \xi}\left[\sum_{i=1}^m \left\|(\nabla f_i ( \theta^{(t)})  -  \nabla F_i( \theta^{(t)}))\right\|^2\right]  = m\sigma^2.
    \end{align*}
    We then analyze the term B. From the proof in Lemma~\ref{lem:dot}, and by the fact that the minus sign does not affect the inequality, we know that
    \begin{align*}
        \text{term B} &= - 2\mathbb{E}_{ \xi} \left[ \left(\sum_{i=1}^m \alpha_{i}^{(t)} \nabla f_i ( \theta^{(t)})\right)^\top \left( \left(\sum_{i=1}^m  \alpha_{i}^{(t)}\nabla F_i ( \theta^{(t)}) \right) +\sum_{i=1}^m \alpha_{i}^{(t)} \nabla f_i ( \theta^{(t)})\right)\right]  \\
        & \leq 4 m^{3/2} L \sigma {\eta_{\alpha}^{(t)}} B_{\ell} . 
    \end{align*}
    Combining all the results, we obtain
    \begin{align*}
        \mathbb{E}_{ \xi} \left[\| \left(\sum_{i=1}^m  \alpha_{i}^{(t)}\nabla F_i ( \theta^{(t)}) \right)\|_2^2 - \|\sum_{i=1}^m \alpha_{i}^{(t)} \nabla f_i ( \theta^{(t)})\|_2^2\right] \leq m\sigma^2 + 4 m^{3/2} L \sigma {\eta_{\alpha}^{(t)}} B_{\ell}.
    \end{align*}
\end{proof}

\begin{lemma}\label{lem:nonconvex_sequence}Under the same assumption as Theorem~\ref{thm:convergence_nonconvex}, select nonincreasing ${\eta_{\theta}^{(t)}} \leq \min\{1/B_\ell,1/Gm\}$, we have the following inequality
    \begin{align*}
        & \frac{{\eta_{\theta}^{(t)}}}{2}\mathbb{E}_{ \xi}\left[\left\|\sum_{i=1}^m  \alpha_{i}^{(t)}\nabla f_i (\theta^{(t)})\right\|_2^2 \right] \\
        & \leq \mathbb{E}_{ \xi} [ \sum_{i=1}^m\alpha_{i}^{(t)} f_i(\theta^{(t)})-\sum_{i=1}^m\alpha_{i}^{(t)} f_i(\theta^{(t+1)})] + 6 m^{3/2} L \sigma {\eta_{\alpha}^{(t)}}{\eta_{\theta}^{(t)}} B_{\ell}  + \frac{G{\eta_{\theta}^{(t)}}^2}{2} m\sigma^2.
    \end{align*}
\end{lemma}
\begin{proof}
    From the $G$-smoothness of each objective function, we have
    \begin{align*}
        \alpha_{i}^{(t)} f_i(\theta^{(t+1)}) \leq \alpha_{i}^{(t)} \left( f_i(\theta^{(t)}) + {\eta_{\theta}^{(t)}} \nabla f_i (\theta^{(t)})^\top  \left(\sum_{i=1}^m  \alpha_{i}^{(t)}\nabla F_i ( \theta^{(t)}) \right) + \frac{G{\eta_{\theta}^{(t)}}^2}{2}\| \sum_{i=1}^m  \alpha_{i}^{(t)}\nabla F_i ( \theta^{(t)}) \|_2^2\right).
    \end{align*}
    Sum up both side for $i=1,\ldots,m$, and take the expectation on random variable $ \xi$, we can get
    \begin{align*}
        & \mathbb{E}_{ \xi} [\sum_{i=1}^m\alpha_{i}^{(t)} f_i(\theta^{(t+1)}) - \sum_{i=1}^m \alpha_{i}^{(t)} f_i(\theta^{(t)})]  \\
        & \leq {\eta_{\theta}^{(t)}} \mathbb{E}_{ \xi} \left[  \left(\sum_{i=1}^m\alpha_{i}^{(t)}\nabla f_i (\theta^{(t)})\right)^\top  \left(\sum_{i=1}^m  \alpha_{i}^{(t)}\nabla F_i ( \theta^{(t)}) \right)\right] + \frac{G{\eta_{\theta}^{(t)}}^2}{2}\mathbb{E}_{ \xi} \left[\sum_{i=1}^m \alpha_{i}^{(t)}\| \sum_{i=1}^m  \alpha_{i}^{(t)}\nabla F_i ( \theta^{(t)}) \|_2^2\right]\\
        & \leq {\eta_{\theta}^{(t)}} \mathbb{E}_{ \xi} \left[  \left(\sum_{i=1}^m\alpha_{i}^{(t)}\nabla f_i (\theta^{(t)})\right)^\top  \left(\sum_{i=1}^m  \alpha_{i}^{(t)}\nabla F_i ( \theta^{(t)}) \right)\right] + \frac{G{\eta_{\theta}^{(t)}}^2}{2}\mathbb{E}_{ \xi} \left[\| \sum_{i=1}^m  \alpha_{i}^{(t)}\nabla F_i ( \theta^{(t)}) \|_2^2\right].
    \end{align*}
    The last inequality is by the Assumption that $\sum_{i=1}^m\alpha_{i}^{(t)}\leq 1$. From the result of Lemma~\ref{lem:dot}, we can bound the first term and obtain
    \begin{align*}
        & \mathbb{E}_{ \xi} [\sum_{i=1}^m\alpha_{i}^{(t)} f_i(\theta^{(t+1)}) - \sum_{i=1}^m\alpha_{i}^{(t)} f_i(\theta^{(t)})] \\ 
        & \leq 4 m^{3/2} L \sigma {\eta_{\alpha}^{(t)}}{\eta_{\theta}^{(t)}} B_{\ell} - {\eta_{\theta}^{(t)}}\mathbb{E}_{ \xi}\left[\left\|\sum_{i=1}^m  \alpha_{i}^{(t)}\nabla f_i (\theta^{(t)})\right\|_2^2 \right] + \frac{G{\eta_{\theta}^{(t)}}^2}{2}\mathbb{E}_{ \xi} \left[\| \sum_{i=1}^m  \alpha_{i}^{(t)}\nabla F_i ( \theta^{(t)}) \|_2^2\right].
    \end{align*}
    Then, adopting the result from Lemma~\ref{lem:second_bias}, we know that
    \begin{align*}
        &  \mathbb{E}_{ \xi} [\sum_{i=1}^m\alpha_{i}^{(t)} f_i(\theta^{(t+1)}) - \sum_{i=1}^m\alpha_{i}^{(t)} f_i(\theta^{(t)})]  \leq 4 m^{3/2} L \sigma {\eta_{\alpha}^{(t)}} {\eta_{\theta}^{(t)}} B_{\ell}  \\
        & \quad\quad + \left(\frac{G{\eta_{\theta}^{(t)}}^2}{2}- {\eta_{\theta}^{(t)}}\right)\mathbb{E}_{ \xi}\left[\left\|\sum_{i=1}^m  \alpha_{i}^{(t)}\nabla f_i (\theta^{(t)})\right\|_2^2 \right]+ \frac{G{\eta_{\theta}^{(t)}}^2}{2} m\sigma^2 + 2 m^{3/2} LG \sigma {\eta_{\theta}^{(t)}}^2 {\eta_{\alpha}^{(t)}} B_{\ell}.
    \end{align*}
    By the fact that ${\eta_{\theta}^{(t)}} \leq 1/Gm$, we further have
    \begin{align*}
        \mathbb{E}_{ \xi} [\sum_{i=1}^m\alpha_{i}^{(t)} f_i(\theta^{(t+1)}) - \sum_{i=1}^m\alpha_{i}^{(t)} f_i(\theta^{(t)})] & \leq 4 m^{3/2} L \sigma {\eta_{\theta}^{(t)}} {\eta_{\alpha}^{(t)}} B_{\ell} + 2 m^{1/2} L \sigma {\eta_{\alpha}^{(t)}}{\eta_{\theta}^{(t)}} B_{\ell} \\
        & \quad - \frac{\eta_{\theta}^{(t)}}{2}\mathbb{E}_{ \xi}\left[\left\|\sum_{i=1}^m  \alpha_{i}^{(t)}\nabla f_i (\theta^{(t)})\right\|_2^2 \right]+ \frac{G{\eta_{\theta}^{(t)}}^2}{2} m\sigma^2  .
    \end{align*}
    By rearrangement, we therefore have
    \begin{align*}
        & \frac{{\eta_{\theta}^{(t)}}}{2}\mathbb{E}_{ \xi}\left[\left\|\sum_{i=1}^m  \alpha_{i}^{(t)}\nabla f_i (\theta^{(t)})\right\|_2^2 \right] \\
        & \leq \mathbb{E}_{ \xi} [ \sum_{i=1}^m\alpha_{i}^{(t)} f_i(\theta^{(t)})-\sum_{i=1}^m\alpha_{i}^{(t)} f_i(\theta^{(t+1)})] + 6 m^{3/2} L \sigma {\eta_{\alpha}^{(t)}}{\eta_{\theta}^{(t)}} B_{\ell}  + \frac{G{\eta_{\theta}^{(t)}}^2}{2} m\sigma^2.
    \end{align*}
    % Take expectation of $\xi_1,\xi_2,\ldots,\xi_n$ on the both sides, we finally get
    % \begin{align*}
    %     & \frac{{\eta_{\theta}^{(t)}}}{2}\mathbb{E}\left[\left\|\sum_{i=1}^m  \alpha_{i}^{(t)}\nabla f_i (\theta^{(t)})\right\|_2^2 \right]  \leq \mathbb{E} [ \sum_{i=1}^m\alpha_{i}^{(t)} f_i(\theta^{(t)})-\sum_{i=1}^m\alpha_{i}^{(t)} f_i(\theta^{(t+1)})]\\
    %     &  \quad\quad\quad\quad\quad\quad\quad\quad\quad\quad + 4m L {\eta_{\theta}^{(t)}}\sqrt{ \mathbb{V}[\lambda_k] \sum_{i=1}^m\mathbb{V} [ \nabla F_i(\theta^{(t)})]}  + \frac{Lm^2B^3{\eta_{\theta}^{(t)}}^2}{2}\sum_{i=1}^m \mathbb{V}( {g}_i) .
    % \end{align*}
\end{proof}

\begin{lemma}\label{lem:nonconvexsummation} Under the same assumption with Theorem~\ref{thm:convergence_nonconvex}, select nonincreasing ${\eta_{\theta}^{(t)}} \leq 1/Gm$, we have
    \begin{align*}
        & \sum_{t=1}^T\frac{1}{{\eta_{\theta}^{(t)}}}\mathbb{E}[ \sum_{i=1}^m\alpha_{i}^{(t)} f_i(\theta^{(t)})-\sum_{i=1}^m\alpha_{i}^{(t)} f_i(\theta^{(t+1)})]  \leq \sum_{t=2}^T \frac{2 m {\eta_{\alpha}^{(t)}} B_{\ell}^2}{{\eta_{\theta}^{(t)}}} + \frac{2mB_{\ell}}{{\eta_{\theta}^{(T)}}}.
    \end{align*}
\end{lemma}
\begin{proof}
    First, we can decompose the left side as
    \begin{align*}
        & \sum_{t=1}^T\frac{1}{{\eta_{\theta}^{(t)}}}\mathbb{E} [ \sum_{i=1}^m\alpha_{i}^{(t)} f_i(\theta^{(t)})-\sum_{i=1}^m\alpha_{i}^{(t)} f_i(\theta^{(t+1)})] \\
        & = \sum_{t=2}^T\mathbb{E} [ \frac{1}{{\eta_{\theta}^{(t)}}}\sum_{i=1}^m\alpha_{i}^{(t)} f_i(\theta^{(t)})-\frac{1}{{\eta_{\theta}^{(t-1)}}}\sum_{i=1}^m\alpha_{i}^{(t-1)} f_i(\theta^{(t)})] + \mathbb{E} [ \frac{1}{\eta_\theta^{(1)}}\sum_{i=1}^m\alpha_{i}^{(1)}  f_i(\theta^{(1)})-\frac{1}{{\eta_{\theta}^{(T)}}}\sum_{i=1}^m\alpha_{i}^{(t)}  f_i(\theta^{(T+1)})]
    \end{align*}
    Then we have the following decomposition for the first term
    \begin{align*}
        & \sum_{t=2}^T\mathbb{E} [ \frac{1}{{\eta_{\theta}^{(t)}}}\sum_{i=1}^m\alpha_{i}^{(t)} f_i(\theta^{(t)})-\frac{1}{{\eta_{\theta}^{(t-1)}}}\sum_{i=1}^m\alpha_{i}^{(t-1)} f_i(\theta^{(t)})] \\
        & = \sum_{t=2}^T\mathbb{E} [ \frac{1}{{\eta_{\theta}^{(t)}}}\sum_{i=1}^m\alpha_{i}^{(t)} f_i(\theta^{(t)})- \frac{1}{{\eta_{\theta}^{(t)}}}\sum_{i=1}^m\alpha_{i}^{(t-1)} f_i(\theta^{(t)})+ \frac{1}{{\eta_{\theta}^{(t)}}}\sum_{i=1}^m\alpha_{i}^{(t-1)} f_i(\theta^{(t)})-\frac{1}{{\eta_{\theta}^{(t-1)}}}\sum_{i=1}^m\alpha_{i}^{(t-1)} f_i(\theta^{(t)}) ] \\
        & = \sum_{t=2}^T\mathbb{E} [ \sum_{i=1}^m\frac{1}{{\eta_{\theta}^{(t)}}}(\alpha_{i}^{(t)}-\alpha_{i}^{(t-1)}) f_i(\theta^{(t)})]+\sum_{t=2}^T\mathbb{E} [ (\frac{1}{{\eta_{\theta}^{(t)}}}-\frac{1}{{\eta_{\theta}^{(t-1)}}} )\sum_{i=1}^m\alpha_{i}^{(t-1)} f_i(\theta^{(t)}) ] \\
        & = \sum_{t=2}^T\frac{2 m {\eta_{\alpha}^{(t)}} B_{\ell}}{{\eta_{\theta}^{(t)}}}\mathbb{E} [  f_i(\theta^{(t)})] +\sum_{t=2}^T\mathbb{E} [ (\frac{1}{{\eta_{\theta}^{(t)}}}-\frac{1}{{\eta_{\theta}^{(t-1)}}} )\sum_{i=1}^m\alpha_{i}^{(t-1)} f_i(\theta^{(t)}) ] 
    \end{align*}
    Combining the above, and by the assumption that $|f_i(\theta^{(t)})| \leq B_{\ell}, \alpha_{i}^{(t)}\in [0,1]$ as well as the nonincreasing for ${\eta_{\theta}^{(t)}}$, we therefore have
    \begin{align*}
        & \sum_{t=1}^T\frac{1}{{\eta_{\theta}^{(t)}}}\mathbb{E}[ \sum_{i=1}^m\alpha_{i}^{(t)} f_i(\theta^{(t)})-\sum_{i=1}^m\alpha_{i}^{(t)} f_i(\theta^{(t+1)})] \\
        & \leq \sum_{t=2}^T\frac{2 m {\eta_{\alpha}^{(t)}} B_{\ell}^2}{{\eta_{\theta}^{(t)}}} +\sum_{t=2}^T\mathbb{E} [ (\frac{1}{{\eta_{\theta}^{(t)}}}-\frac{1}{{\eta_{\theta}^{(t-1)}}} )\sum_{i=1}^m\alpha_{i}^{(t-1)} |f_i(\theta^{(t)})| ] \\
        & \quad + \mathbb{E} [ \frac{1}{\eta_\theta^{(1)}}\sum_{i=1}^m\alpha_{i}^{(1)} | f_i(\theta^{(1)})|+\frac{1}{{\eta_{\theta}^{(T)}}}\sum_{i=1}^m\alpha_{i}^{(t)}  |f_i(\theta^{(T+1)})|]\\
        & \leq \sum_{t=2}^T\frac{2 m {\eta_{\alpha}^{(t)}} B_{\ell}^2}{{\eta_{\theta}^{(t)}}} +mB_{\ell}\sum_{t=2}^T\mathbb{E} [ (\frac{1}{{\eta_{\theta}^{(t)}}}-\frac{1}{{\eta_{\theta}^{(t-1)}}} ) + \frac{mB_{\ell}}{\eta_\theta^{(1)}} + \frac{mB_{\ell}}{{\eta_{\theta}^{(T)}}} \\
        & \leq \sum_{t=2}^T \frac{2 m {\eta_{\alpha}^{(t)}} B_{\ell}^2}{{\eta_{\theta}^{(t)}}} + \frac{2mB_{\ell}}{{\eta_{\theta}^{(T)}}}.
    \end{align*}
\end{proof}
Then plug into Lemma~\ref{lem:nonconvex_sequence}, we get
    \begin{align*}
        & \mathbb{E}\left[\left\|\sum_{i=1}^m  \alpha_{i}^{(t)}\nabla f_i (\theta^{(t)})\right\|_2^2 \right]  \\
        & \leq \frac{2}{{\eta_{\theta}^{(t)}}}\mathbb{E} [ \sum_{i=1}^m\alpha_{i}^{(t)} f_i(\theta^{(t)})-\sum_{i=1}^m\alpha_{i}^{(t)} f_i(\theta^{(t+1)})]+ 12 m^{3/2} L \sigma {\eta_{\alpha}^{(t)}} B_{\ell}  + {G{\eta_{\theta}^{(t)}}}m\sigma^2.
    \end{align*}
    The above is divided with ${\eta_{\theta}^{(t)}}/2$ for the both side. Then, sum up the above inequality from $t = 1$ to $t=T$ and take the result from Lemma~\ref{lem:nonconvexsummation}, we have
    \begin{align*}
        &  \sum_{t=1}^T\mathbb{E}\left[\left\|\sum_{i=1}^m  \alpha_{i}^{(t)}\nabla f_i (\theta^{(t)})\right\|_2^2 \right]  \\
        & \leq \sum_{t=1}^T\frac{2}{{\eta_{\theta}^{(t)}}}\mathbb{E} [ \sum_{i=1}^m\alpha_{i}^{(t)} f_i(\theta^{(t)})-\sum_{i=1}^m\alpha_{i}^{(t)} f_i(\theta^{(t+1)})]+ 12 m^{3/2} L \sigma B_{\ell} \sum_{t=1}^T {\eta_{\alpha}^{(t)}}   + Gm\sigma^2\sum_{t=1}^T {\eta_{\theta}^{(t)}}.\\
        & \leq \sum_{t=2}^T \frac{4 m {\eta_{\alpha}^{(t)}} B_{\ell}^2}{{\eta_{\theta}^{(t)}}} + \frac{4mB_{\ell}}{{\eta_{\theta}^{(T)}}}+ 12 m^{3/2} L \sigma B_{\ell} \sum_{t=1}^T {\eta_{\alpha}^{(t)}}   + Gm\sigma^2\sum_{t=1}^T {\eta_{\theta}^{(t)}}.
    \end{align*}
    By selecting $\eta_{\theta}^{(t)} = 2\sqrt{B_\ell}/\sigma \sqrt{TG}$, and $ \eta_{\alpha}^{(t)} \leq \min \{\sqrt{G}/2B_\ell T, \sqrt{mG}/12L\sqrt{B_{\ell}T}\}$. Then average the above inequality, we prove the theorem.

\section{More Experimental Details and Results}\label{supp:exp}

\subsection{Implementation}

We implement our algorithm using hard parameter sharing, where all tasks share a feature extractor and have task-specific heads. For feature extractors, we use a two-layer CNN for MultiMNIST, ResNet-18~\citep{he2016deep} for Office-Home, and SegNet~\citep{badrinarayanan2017segnet} for NYUv2. For task-specific heads, we use two-layer CNNs for NYUv2, and MLP for all other datasets. We standardize all datasets to ensure zero mean and unit variance, as excess risks are sensitive to the scale of tasks. The details are as follows.

It is a well-known fact that overparametrized models can achieve 0 training error even from a dataset with pure random noise. This means that the training loss will always decrease, but not plateau, even if substantial label noise is contained. To address this issue, we use weight decay in the experiments on SARCOS, Office-Home and NYUv2, as we employ overparametrized models on them. We use Adam optimizer and ReLU activation on all datasets. For easier direct comparisons across different model types, we use constant learning rates instead of adaptive ones. The experiments are run on NVIDIA RTX A6000 GPUs.

For all datasets except NYUv2, we use linear layers as task-specific heads. On MultiMNIST, we use a two-layer CNN with kernel size 5 followed by one fully connected layer with 80 hidden units as the feature extractor, trained with learning rate 1e-3. Since the model size is small, we do not apply any regularization. On Office-Home, we use a ResNet 18 (without pretraining) as the shared feature extractor, which is trained using a weight decay of 1e-5. The learning rate is 1e-4. On SARCOS, we use a three-layer MLP with 128 hidden units as the shared feature extract, which is also trained using a weight decay of 1e-2. The learning rate is 1e-3.
% On CelebA, we use a ResNet 18 (without pretraining) as the shared feature extractor, which is trained using a weight decay of 1e-2. The learning rate is 1e-4. 

On NYUv2, we follow the implementation of~\citet{liu2019end}. Since the dataset have limited data, making it prone to overfitting, we use data augmentation as suggested by \citet{liu2019end}. For the feature extractor, we use the SegNet architecture with four convolutional encoder layers and four convolutional decoder layers. For each of the three tasks, we use two convolutional layers as the task-specific heads. We use a weight decay of 1e-3 and the learning rate is 1e-4.

When performing scale processing as mentioned in \Cref{algo}, directly dividing by the initial excess risk can be unstable since it heavily depends on the initialization of the network, which can be arbitrarily poor. To address this, we deploy a warm-up strategy, where we do not do weight update in the first 3 epochs to collect the average risks over those epochs as an estimation of the initial excess risk.

For the implementation of baselines, we use the code from \citet{lin2023libmtl} and \citet{nashmtl}. 

\subsection{More Results}
\begin{figure}[t!]
    \centering
    \includegraphics[width=\linewidth]{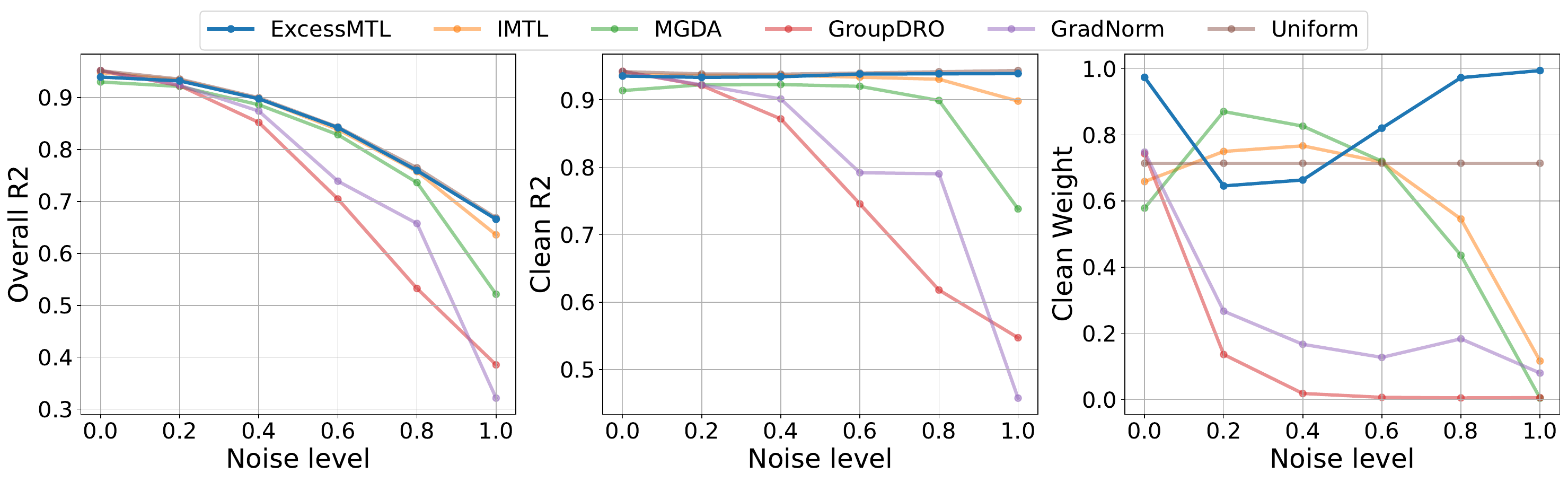}
    \caption{Results on the SARCOS dataset (noise in the first two joints). The left figure considers all 7 tasks, while the other two considers all tasks except the first two tasks. The right figure is the combined weights of all clean tasks (around 0.71 for uniform scalarization). ExcessMTL consistently maintains its performance, significantly outperforming other adaptive methods in face of label noise.}
    \label{fig:sarcos}
\end{figure}

\textbf{SARCOS~\citep{vijayakumar2000locally}} presents an inverse dynamics problem for a robot arm with seven degrees of freedom. The task is to perform multi-target regression that uses 21 attributes (7 joint positions, 7 joint velocities, 7 joint accelerations) to predict the corresponding 7 joint torques. The noise is injected into the first two joint torques.

The results on SARCOS is presented in \Cref{fig:sarcos}. In the face of increasing label noise, all adaptive weighting algorithms except ExcessMTL exhibit a trend of assigning increasing weights to the noisy tasks. This behavior leads to a decline in performance on clean tasks. ExcessMTL demonstrates resilience to label noise, consistently maintaining its performance. Here, a similar pattern to MultiMNIST can be observed that uniform scalarization performs well. However, we want to emphasize again that the performance of uniform scalarization varies across all datasets, and it is not able to produce consistent results universally. In contrast, ExcessMTL demonstrates consistency across diverse datasets, reinforcing its robustness and effectiveness as a reliable solution in noisy environments.

\end{document}